\documentclass[12pt]{article}
\title{Efficiency for Regularization Parameter Selection in Penalized Likelihood Estimation of Misspecified Models}
\author{Cheryl J. Flynn, Clifford M. Hurvich, and Jeffrey S. Simonoff \\
    New York University}
\date{\today}

\usepackage[margin=2.5cm]{geometry}
\usepackage{amsmath}
\usepackage{amsthm}
\usepackage{enumitem}
\usepackage{graphics}
\usepackage{float}
\usepackage{amssymb}
\usepackage{natbib}
\usepackage{subfigure}
\usepackage{lscape}
\usepackage{setspace}
\usepackage{macros}

\bibpunct{(}{)}{;}{a}{,}{,}
\DeclareMathOperator{\Var}{Var}

\begin{document}

\newtheorem{mydef}{Definition}[section]
\newtheorem{thm}{Theorem}[section]

\newcommand{\Keywords}[1]{\par\noindent
{\small{KEY WORDS}: #1}}

\maketitle

\begin{abstract}
It has been shown that $AIC$-type criteria are asymptotically efficient selectors of the tuning parameter in non-concave penalized regression methods under the assumption that the population variance is known or that a consistent estimator is available.  We relax this assumption to prove that $AIC$ itself is asymptotically efficient and we study its performance in finite samples.  In classical regression, it is known that $AIC$ tends to select overly complex models when the dimension of the maximum candidate model is large relative to the sample size.  Simulation studies suggest that $AIC$ suffers from the same shortcomings when used in penalized regression.  We therefore propose the use of the classical corrected $AIC$ ($AIC_c$) as an alternative and prove that it maintains the desired asymptotic properties.  To broaden our results, we further prove the efficiency of $AIC$ for penalized likelihood methods in the context of generalized linear models with no dispersion parameter.  Similar results exist in the literature but only for a restricted set of candidate models.  By employing results from the classical literature on maximum-likelihood estimation in misspecified models, we are able to establish this result for a general set of candidate models.  We use simulations to assess the performance of $AIC$ and $AIC_c$, as well as that of other selectors, in finite samples for both SCAD-penalized and Lasso regressions and a real data example is considered.

\end{abstract}

\section{Introduction}

Regularized (or penalized) likelihood methods have become widely used in recent years due to the increased availability of large data sets.  These methods operate by maximizing the penalized likelihood function
\begin{equation}\label{penLik}
\frac{1}{n}l(\boldsymbol\beta) - \sum_{j=1}^{d_n} p_\lambda (|\beta_j|)
\end{equation}
with respect to $\boldsymbol{\beta} \in \mathbb{R}^{d_n}$, where $l(\beta)$ is the working log-likelihood function, $d_n$ is the total number of predictors, and $p_\lambda (\cdot)$ is a penalty function that penalizes against model complexity and the size of the estimated coefficients.  The working log-likelihood is used to justify the first part of the function (e.g., in Least Squares, the working log-likelihood is based on the Gaussian distribution).  As demonstrated in Sections 2 and 3, many of the results discussed in this paper are valid even if the working log-likelihood is misspecified.  With these methods, increasing the amount of regularization increases the number of estimated coefficients that are set equal to zero thus performing ``automatic'' variable selection through the data-dependent choice of the regularization parameter, $\lambda$.  In contrast, variable selection in classical regression is commonly done using the Leaps and Bounds algorithm \citep{furnival74}, which becomes infeasible when the number of predictors is much larger than 30 \citep{hastie09}.  For most penalty functions efficient algorithms exist to compute the estimated models over a regularization path making it possible to do variable selection in high dimensions.

The performance of the estimated model heavily depends on the choice of the regularization parameter.  In regularized regression several classical model selection procedures have been heuristically applied as selectors of this parameter including information criteria such as Akaike's information criterion ($AIC$; \citealp{akaike73}), the Bayesian information criterion ($BIC$; \citealp{schwarz78}), and Generalized cross-validation ($GCV$; \citealp{craven78}) as well as data-based selection procedures such as $k$-fold cross-validation (see, e.g., \citealp{fan01}, \citealp{zou07}, \citealp{wang07}, and \citealp{zhang10} for applications of these selectors to penalized regression estimators).  The statistical properties of these model selection procedures have been widely studied in the context of classical regression and an ongoing research problem is to determine if these properties carry over to the context of penalized regression.

The asymptotic performance of model selection procedures can be studied under two important and distinct settings: (1) when the true model is not among the candidate models (the ``non-true model world'') and (2) when the true model is among the candidate models (the ``true model world'').  In the non-true model world a reasonable goal is \textit{efficient} model selection, meaning that we would like to select the model that asymptotically performs the best amongst the candidate models.  In contrast, in the true-model world most of the literature focuses on \textit{consistent} model selection, meaning that the probability that the true model is chosen is asymptotically one. In general, a model selection procedure cannot be both consistent and efficient \citep{shao97,yang05}.  Although the non-true model world has been extensively studied in classical regression (e.g., \citealp{shibata81}, \citealp{li87}, \citealp{hurvich89,hurvich91}, \citealp{shao97}, and \citealp{burnham02}) the majority of the research on model selection in penalized regression has focused on the true model world (e.g., \citealp{leng06}, \citealp{zou07}, and \citealp{wang07}).  We feel that the non-true model world is more realistic in many situations since the data-generating process is likely to be too complex to know exactly; this is the essence of George Box's famous admonition that ``all models are wrong, but some are useful'' \citep{box79}.  This setting should be of particular interest to researchers and data analysts in areas such as social science and environmental health where a large number of predictors are expected to influence the dependent variable (too many to include in model fitting; \citealp{gelman10}) as well as machine learning where the goal is typically not to uncover the true data generating process but rather to find a model that can predict well.

In the context of generalized linear models (GLMs), \citet{zhang10} (hereafter ZLT) proposed the use of a ``GIC-type'' criterion,
\[
GIC_{\kappa_n} =
    -\frac{1}{n}l(\hat{\boldsymbol\beta}_\lambda) + \kappa_n \frac{df_\lambda}{n}
\]
for choosing the regularization parameter $\lambda$ for non-concave penalized estimators in both the non-true model world and the true-model world.  Here $\hat{\boldsymbol\beta}_\lambda$ is the estimator that maximizes (\ref{penLik}) for a specific $\lambda$, $df_\lambda$ is the effective degrees of freedom and the log-likelihood function corresponds to a member of the exponential family, i.e.
\[
l(\hat{\boldsymbol\beta}_\lambda)
       = \sum_{i=1}^n
            \left(\frac{y_i \mathbf{x}_i^T \hat{\boldsymbol\beta}_\lambda -
            b(\mathbf{x}_i^T \hat{\boldsymbol\beta}_\lambda)}
            {a(\phi)}
            + c(y_i,\phi) \right),
\]
where the form of functions $a(\cdot)$, $b(\cdot)$, and $c(\cdot, \cdot)$ depends on the specified distribution and $\phi$ is the dispersion parameter (see e.g. \citealp{mccullagh89}).  They showed that ``AIC-type'' versions of $GIC_{\kappa_n}$ ($\kappa_n \rightarrow 2$) are efficient in the former case, while ``BIC-type'' versions of $GIC_{\kappa_n}$ ($\kappa_n \rightarrow \infty$ and $\kappa_n / \sqrt{n} \rightarrow 0$) are consistent in the latter case.

In the Gaussian model, $GIC_{\kappa_n}$ takes on a form that includes the true error variance $\sigma^2$, and the proofs operate under the assumption that this is known or that a consistent estimator is available.  However, if the true model is not included in the set of candidate models then a consistent estimator of the true error variance may not be available \citep{shao97} making the efficiency proofs of ZLT not applicable in practice.  This motivates us to extend the ZLT results in various ways.  First, we show that the feasible version of $GIC_{2}$, which corresponds to the well-known $C_p$ measure \citep{mallows73}, is in fact efficient in the non-true model world.  Second, we show that $AIC$ and $GCV$, which do not require a consistent estimator of $\sigma^2$, are also efficient.  Third, we show that although several model selection procedures may be asymptotically optimal, performance varies in finite samples.  Specifically, we study performance when the number of predictors is allowed to be large relative to the sample size and show that $AIC$, $BIC$, $C_p$, and $GCV$ all have a tendency to sometimes catastrophically overfit (lead to $\lambda$ values approaching 0).  In classical regression \citet{hurvich89} showed that $AIC$ has a tendency to select overly complex models when the dimension of the maximum candidate model is large relative to the sample size and proposed a corrected version of $AIC$ ($AIC_c$).  We show that $AIC_c$ is also efficient, but avoids the tendency to select overly complex models.  We use Monte Carlo simulations to illustrate the properties of these methods in finite samples and compare their performance against the data-dependent method 10-fold $CV$.

For GLMs where there is no dispersion parameter (e.g., probit and logistic regression or the Poisson log-linear model), there is no difference between $GIC_{2}$ and $AIC$.  However, in their proof ZLT restrict the set of candidate models to ones where the estimated parameter converges in probability to the true parameter uniformly.  To weaken this assumption we employ the result from \citet{white82} that the maximum-likelihood estimator converges almost surely to a ``pseudo-true'' parameter (the parameter that minimizes the Kullback-Leibler (KL) loss function) when the model is misspecified and prove the efficiency of $AIC$ under a weaker set of assumptions.  These results, and the results for the Gaussian model, apply to a wide range of penalized likelihood estimators, including both non-concave penalized estimators and the well-known Least absolute shrinkage and selection operator (Lasso) estimator \citep{tibshirani96}.

The remainder of the paper is organized as follows.  Section 2 focuses on penalized regression and establishes the efficiency results for $C_p$, $AIC$, $GCV$ and $AIC_c$ without the assumption that the true population variance is known or that a consistent estimator exists.  Section 3 focuses on GLMs where there is no dispersion parameter and establishes the efficiency of $AIC$ for a general set of candidate models.  Section 4 presents simulation results that explore the finite-sample behavior of the different selectors when the number of predictors is allowed to be large relative to the sample size.  An empirical example that highlights the varying performance of the selectors is presented in Section 5.  Concluding remarks are given in Section 6.  The main proofs are included in the appendix with some auxiliary results included in the supplementary material.

\section{Gaussian Model}
For ease of notation, in this section, and for the remainder of the paper, we suppress the subscript $n$ where we feel it is clear that a variable depends on the sample size.

To study model selection in regularized regression we consider the model
\begin{equation*}
\mathbf{y} = \boldsymbol{\mu} + \boldsymbol{\varepsilon},
\end{equation*}
where $\mathbf{y}=(y_1,\ldots,y_n)^T$ is the $n\times 1$ response vector, $\boldsymbol{\mu}=(\mu_1,\ldots,\mu_n)^T$ is a $n\times 1$ unknown mean vector and the entries of the $n\times 1$ error vector $\boldsymbol\varepsilon$ are independent and identically distributed (iid) with mean 0 and variance $\sigma^2$.  The mean vector is estimated by $\hat{\boldsymbol{\mu}}_\lambda=\mathbf{X}\hat{\boldsymbol{\beta}}_\lambda$ where $\mathbf{X}=(
\mathbf{x}_1,\ldots,\mathbf{x}_n)^T$ is a $n\times d_n$ deterministic matrix of predictors and $\hat{\boldsymbol\beta}_\lambda$ is the estimator that minimizes the penalized least squares function
\begin{equation*}
\frac{1}{n}\sum_{i=1}^n {(y_i - \mathbf{x}_{i}\boldsymbol{\beta})^2} + \sum_{j=1}^{d_n} {p_\lambda(|\beta_j|)}
\end{equation*}
with respect to $\boldsymbol{\beta} \in \mathbb{R}^{d_n}$.

Adopting the notation from ZLT, we let the index set $\mathcal{A}_n$ denote the class of all candidate models and we assume that $\bar{\alpha}=\{1,\ldots,d_n\}$ is the largest model in $\mathcal{A}_n$.  For any $\alpha \in \mathcal{A}_n$, we define $d_\alpha$ to be the number of predictor variables included in the candidate model.  We further define the least squares estimated mean vector by $\hat{\boldsymbol\mu}_\alpha = \mathbf{X}_\alpha \hat{\boldsymbol\beta}_\alpha$ where $\mathbf{X}_\alpha$ is the matrix of predictors that are included in candidate model $\alpha$ and $\hat{\boldsymbol{\beta}}_\alpha$ is the corresponding vector of the estimated least squares coefficients. The associated projection matrix is $\mathbf{H}_\alpha=\mathbf{X}_\alpha(\mathbf{X'}_\alpha\mathbf{X}_\alpha)^{-1}
\mathbf{X'}_\alpha$.  For a given $\lambda$, we define $\alpha_\lambda$ to be the model $\alpha \in \mathcal{A}_n$ whose predictors are those with non-zero coefficients in the penalized estimator $\hat{\boldsymbol\beta}_\lambda$ and let $df_\lambda$ denote the effective degrees of freedom. The least squares estimated mean vector based on the model $\alpha_\lambda$ is denoted by $\hat{\boldsymbol{\mu}}_{\alpha_\lambda} = \mathbf{X}_{\alpha_\lambda}\hat{\boldsymbol{\beta}}_{\alpha_\lambda}$.  In this equation, $\mathbf{X}_{\alpha_\lambda}$ is the matrix of predictors whose coefficients are not shrunk to zero in the penalized estimator $\hat{\boldsymbol\beta}_\lambda$ and $\hat{\boldsymbol{\beta}}_{\alpha_\lambda}$ are the estimated coefficients from the least squares model fit using these predictors.  The associated projection matrix in this case is defined as $\mathbf{H}_{\alpha_\lambda}=\mathbf{X}_{\alpha_\lambda}(\mathbf{X'}_{\alpha_\lambda}\mathbf{X}_{\alpha_\lambda})^{-1}
\mathbf{X'}_{\alpha_\lambda}$.

If we assume that we are in the non-true model world, then a reasonable goal is efficient model selection.  The $L_2$ loss is commonly used to assess the predictive performance of an estimator and is calculated as
\[ L(\hat{\boldsymbol\beta}_\lambda) = \frac{||\boldsymbol\mu - \hat{\boldsymbol\mu}_\lambda||^2}{n}. \]
If we let $\hat{\lambda}_n$ denote the regularization parameter selected by a given selection procedure, then the procedure is defined to be \textit{asymptotically loss efficient} if
\[ \frac{L(\hat{\boldsymbol\beta}_{\hat{\lambda}_n})}{\inf_{\lambda \in [0,\lambda_{max}]} L(\hat{\boldsymbol\beta}_{\lambda})} \rightarrow_p 1\]
and $\hat{\boldsymbol\beta}_{\hat{\lambda}_n}$ is said to be an \textit{asymptotically loss efficient estimator}.

For the efficiency proofs we further require the following notation.  In classical regression the risk function is defined as
\begin{equation*}
R(\hat{\boldsymbol\beta}_\alpha) =  \E_0 \left (\frac{||\boldsymbol\mu - \hat{\boldsymbol\mu}_\alpha||^2}{n} \right )
 = \Delta_{\alpha} + \frac{\sigma^2 d_\alpha}{n},
 \end{equation*}
where $E_0$ denotes expectation under the true model and $ \Delta_\alpha = ||\boldsymbol\mu - \mathbf{H}_\alpha\boldsymbol\mu||^2/n$.  Letting $d_{\alpha_\lambda}$ denote the number of predictors with non-zero coefficients in the penalized estimator $\hat{\boldsymbol\beta}_\lambda$, we further define the function
\begin{equation*}
\tilde{R}(\hat{\boldsymbol\beta}_{\alpha_\lambda})
 = \Delta_{\alpha_\lambda} + \frac{\sigma^2 d_{\alpha_\lambda}}{n},
 \end{equation*}
which is a random variable.

\subsection{Model Selection Procedures} \label{sec:Procedures}
$K$-fold $CV$ is commonly used to select tuning parameters in both the statistical and machine learning literature.  It operates by first randomly dividing the data set into $k$ roughly equally sized subsets, then for each subset, the prediction error is computed based on the model fit using the data excluding that subset.  The tuning parameter that minimizes the average square error computed across the subsets is then selected.  In classical regression it has been shown that $k$-fold $CV$ should have the same asymptotic properties as $GIC_{\kappa_n}$ with
\begin{equation*}
\kappa_n = \frac{2k-1}{k-1}
\end{equation*}
\citep{shao97}.  Applying this result, 10-fold $CV$ should have the same asymptotic performance as $GIC_{\kappa_n}$ with $\kappa_n=2.\overline{11}$, suggesting that 10-fold $CV$ should be efficient.
Under the assumption of an orthonormal design matrix \citet{leng06} showed that if the Lasso-estimated model minimizes the prediction error then it will fail to select the true model with non-zero probability.  The authors noted that this suggests that $k$-fold $CV$ is inconsistent, but to our knowledge, the asymptotic properties of $k$-fold $CV$ have not been fully established in the context of penalized regression.  While a rigorous extension of the classical theory for $k$-fold $CV$ to penalized regression is beyond the scope of this paper, the simulation results suggest that the k-fold $CV$ is efficient in the current context.

In addition to 10-fold CV, we study the performance of several information criteria.  Specifically, we consider
\begin{equation*}
AIC_\lambda = \log(\hat{\sigma}^2_\lambda) + 2\frac{df_\lambda}{n},
\end{equation*}
\begin{equation*}
AIC_{c_\lambda} = \log(\hat{\sigma}^2_\lambda) + 2\frac{df_\lambda + 1}{n-df_\lambda-2},
\end{equation*}
\begin{equation*}
BIC_\lambda = \log(\hat{\sigma}^2_\lambda) + \log(n)\frac{df_\lambda}{n},
\end{equation*}
\begin{equation*}
GCV_\lambda = \frac{\hat{\sigma}^2_\lambda}{(1-df_\lambda/n)^2},
\end{equation*}
and
\begin{equation*}
C_{p_\lambda} = \hat{\sigma}^2_\lambda + 2\frac{df_\lambda \tilde{\sigma}^2}{n}.
\end{equation*}
In the above we define
\begin{equation*}
\hat{\sigma}^2_\lambda = \frac{||\mathbf{y}-\mathbf{X}\hat{\boldsymbol\beta}_\lambda||^2}{n}
\end{equation*}
and
\begin{equation*}
\tilde{\sigma}^2 = \frac{ ||\mathbf{y} - \mathbf{X}\hat{\boldsymbol{\beta}}_{\bar{\alpha}}||^2 }{n-d_n-1}.
\end{equation*}
With the exception of 10-fold CV, all of the above model selection procedures require a definition of the effective degrees of freedom for the penalized regression method.  In what follows, we use a heuristic definition and define the effective degrees of freedom to be the number of non-zero coefficients in $\hat{\boldsymbol\beta}_\lambda$ and denote this by $d_{\alpha_\lambda}$.  \citet{zou07} proved that the number of non-zero coefficients is an unbiased estimator of the degrees of freedom for the Lasso.  For SCAD, \citet{fan01} proposed setting the degrees of freedom equal to the trace of the approximate linear projection matrix.  Based on Proposition 1 from ZLT, our efficiency proofs would still hold if this alternate definition is used.

\subsection{Efficiency Results}
We show here that assuming that the true model is not in the set of candidate models, $C_{p_\lambda}$, $AIC_\lambda$, $GCV_\lambda$, and $AIC_{c_\lambda}$ are efficient selectors of the regularization parameter.  The dimension of the full model, $d_n$, is allowed to tend to infinity with $n$ but it is assumed that $d_n/n \rightarrow 0$.  The efficiency proofs operate under the same assumptions as those of ZLT, which are presented here for completeness:
\begin{enumerate}
\item [(A1)] $( \frac{1}{n} \mathbf{X'}\mathbf{X})^{-1}$ exists and its largest eigenvalue is bounded by a constant number C.
\item[(A2)] $E\varepsilon_1^{4q}<\infty$, for some positive integer $q$.
\item[(A3)] The risks of the least squares estimators $\hat{\boldsymbol{\beta}}_\alpha$ satisfy
\[ \sum_{\alpha \in \mathcal{A}_n} (nR(\hat{\boldsymbol\beta}_\alpha))^{-q} \rightarrow 0.\]
\item[(A4)] \[ \sup_{\lambda \in [0,\lambda_{max}]} \frac{ ||\mathbf{b} ||^2}{\tilde{R}(\hat{\boldsymbol\beta}_{\alpha_\lambda})} \rightarrow_p 0,\]
where $\mathbf{b}$ is a $d_n \times 1$ vector where $b_i = p'_\lambda(|\hat{\beta}_{\lambda i}|)sgn(\hat{\beta}_{\lambda_i})$ for all $i$ such that $|\hat{\beta}_{\lambda i}|>0$ and is equal to $0$ otherwise.
\end{enumerate}

The first three assumptions are common in the literature on model selection.  Assumption (A1) requires the matrix of predictors to have full column rank and (A2) implies that efficiency can still apply even when penalized least squares is used but the true distribution of the error terms is not Gaussian.  Assumption (A3) puts a restriction on how close the candidate models can be to the true model and precludes any scenario where the true model is included in the set of candidate models.  The last assumption, (A4), is the only assumption that involves the penalty function and ZLT provided the following three sufficient conditions for the assumption to be satisfied.

\begin{enumerate}
\item[(S1)] $\sqrt{n} \lambda_{\max} < M_1$ for all $n$ for some constant $M_1>0$.
\item[(S2)] For any $\theta$, $p'(\theta) \leq M_2 \lambda$ for some constant $M_2>0$.
\item[(S3)] $n|| \vmu - \mH_{\bar{\alpha}} \vmu||^2/d_n \rightarrow \infty$ as $n \rightarrow \infty$.
\end{enumerate}

As pointed out by an anonymous referee, assumption (A3) restricts the size of the set of candidate models.  The classical literature on model selection primarily worked with nested subsets and did not require the consideration of all subsets (e.g., \citet{shibata81}, \citet{shao97}, and \citet{li87}); however, since the subsets selected by methods such as the Lasso or SCAD are data dependent, the set of candidate models is random and we cannot rule out any particular candidate model a priori. Therefore we need $\mathcal{A}_n$ to include all $2^{d_n}$ subsets in order to use the theory from classical model selection.  Alternatively, if the data analyst can assume that the error terms are normally distributed then assumption (A3) can be replaced by a weaker assumption from \citet{shibata81}.

\begin{itemize}
\item[(A3$^*$)] For any $0<\delta<1$, $\sum_{\alpha \in \mathcal{A}_n} \delta^{n R(\hat\beta_\alpha)} \to 0 $,
\end{itemize}

The following lemma details the restrictions on the behavior of $d_n$.

\begin{lemma}\label{A3_lemma}
Assume that for all $n$ sufficiently large
\begin{equation}\label{bias_bound}
|| \vmu - \mH_{\bar{\alpha}} \vmu||^2 \geq k_1 n d_n^{k_2}
\end{equation}
for some positive constant $k_1$ and some constant $k_2 \leq 0$.
Then (A3) will hold if
\begin{equation}\label{A3_cond}
\lim_{n \to \infty} \frac{d_n}{\log_2(n)} < q,
\end{equation}
and (A3$^*$) will hold if
\begin{equation}\label{A3'_cond}
\lim_{n \to \infty} n d_n^{k_2-1} = \infty.
\end{equation}
\end{lemma}
\noindent The proof is presented in the appendix.  This lemma shows that under (A3) $d_n$ can at most grow logarithmically with $n$; however, polynomial growth rates are allowed under assumption (A3$^*$) so long as $d_n = n^c$ for $c<\frac{1}{1-k_2}$.  Specific values of $k_2$ are worked out for the simulation examples considered in Section 4.1.

The asymptotic efficiency of $C_{p_\lambda}$ is given by the following result.
\begin{thm}
Assuming (A1)-(A4) hold and that $d_n/n \rightarrow 0$ as $n \rightarrow \infty$, the regularization parameter, $\hat{\lambda}_n$, selected by minimizing $C_{p_\lambda}$ yields an asymptotically loss efficient estimator, $\hat{\beta}_n(\hat{\lambda}_n)$.
\end{thm}
To further establish the efficiency of $AIC_\lambda$, $GCV_\lambda$ and $AIC_{c_\lambda}$ we require the following two theorems. The first proves the efficiency of $GIC_\lambda$ with the true error variance replaced by the estimated error variance based on the candidate model.
\begin{thm}\label{thm2}
Assuming (A1)-(A4) hold and that $d_n/n \rightarrow 0$ as $n \rightarrow \infty$, the regularization parameter, $\hat{\lambda}_n$, selected by minimizing
\[ \Gamma_n(\lambda) = \hat{\sigma}^2_\lambda \left( 1 + \frac{2d_{\alpha_\lambda}}{n} \right)\]
yields an asymptotically loss efficient estimator, $\hat{\beta}_{\hat{\lambda}_n}$.  The same result holds under normality of the error terms with (A3$^*$) replacing (A3).
\end{thm}
Next, we prove that any procedure that is asymptotically equivalent to $\Gamma_n(\lambda)$ is also efficient.
\begin{thm}\label{thm3}
Assuming (A1)-(A4) hold and that $d_n/n \rightarrow 0$ as $n \rightarrow \infty$, any information criterion that can be written in the form
\begin{equation*}
\tilde{\Gamma}_\lambda = \hat{\sigma}^2_\lambda \left(1+\frac{2d_{\alpha_\lambda}}{n} + \delta_\lambda\right),
\end{equation*}
where
\begin{equation}\label{assump1}
\sup_{\lambda \in [0,\lambda_{max}]} |\delta_\lambda| \rightarrow_p 0 \tag{C1}
\end{equation}
and
\begin{equation}\label{assump2}
\sup_{\lambda \in [0,\lambda_{max}]} \frac{|\delta_\lambda|}{L(\hat{\beta}_\lambda)} \rightarrow_p 0, \tag{C2}
\end{equation}
is an asymptotically loss efficient procedure for selecting $\lambda$.  The same result holds under normality of the error terms with (A3$^*$) replacing (A3).
\end{thm}
Condition (\ref{assump2}) in Theorem \ref{thm3} is a stronger assumption than in the analogous result established by Theorem 4.2 in \citet{shibata80} for selecting the optimal order of a linear process, but Theorem \ref{thm3} is sufficient to show that $AIC_\lambda$, $GCV_\lambda$, and $AIC_{c_\lambda}$ are asymptotically loss efficient model selection procedures for the regularization parameter.  All three methods can be shown to satisfy  (\ref{assump1}) and (\ref{assump2}) using Taylor series expansions.  The details are provided in the supplementary material.

\textit{Remark 1.} The efficiency proofs in this section make use of the results from \citet{li87}, which operate under assumptions (A1)-(A3).  Similar results exist in \citet{shibata81} if the error terms are normally distributed and (A3$^*$) is substituted for (A3).  The efficiency of $AIC_\lambda$, $AIC_{c_\lambda}$, and $GCV_\lambda$ can be shown in a similar manner in this setting.

\section{GLMs with No Dispersion Parameter}
We now generalize our efficiency results to a broader class of models by studying the asymptotic performance of $AIC_\lambda$ as a selector of $\lambda$ when the likelihood function is misspecified as a generalized linear model (GLM) and prove that it is asymptotically loss efficient.  We assume that the data $y_1, \ldots, y_n$ are independent with common unknown probability density function $g(y)$ and that $\E(y_i) = \mu_i$ and $\Var(y_i) = \sigma^2_i$.   To approximate this distribution, we consider a family of GLMs where the density of each candidate model is given by
\[
f_\alpha(y_i;\boldsymbol\beta_\alpha) = \exp \left(y_i \boldsymbol\theta_{\alpha i}
                                    - b(\boldsymbol\theta_{\alpha i})
                                    + c(y_i) \right),
\]
where $\boldsymbol\theta_\alpha = \mathbf{X}_\alpha \boldsymbol\beta_\alpha$, for $\alpha \in \mathcal{A}_n$.  Here we have assumed that there is no dispersion parameter, and we further assume that $b(\theta)$ is three times differentiable and that $b''(\theta)>0$ for all $\theta$.  All of these assumptions would hold for probit or logistic regression and the Poisson log-linear model.

A reasonable objective in this setting is to minimize two times the average Kullback-Leibler (KL) loss function, which is defined as
\[
L_{KL}(\boldsymbol\beta_\alpha)
    =  \frac{2}{n} \sum_{i=1}^n \E_0 \left(\log g(y_i)\right) - \E_0 \left( \log f_\alpha(y_i;\boldsymbol\beta_\alpha) \right)
    =  \frac{2}{n} \sum_{i=1}^n \left[ \mu_i (\vtheta_{0i} - \vtheta_{\alpha i}) +  ( b(\vtheta_{\alpha i}) - b(\vtheta_{0i})) \right].
\]
For a given sample size $n$, we define $\boldsymbol\theta^*_\alpha = \mathbf{X}_\alpha \boldsymbol\beta^*_\alpha$ as the minimizer of the KL loss.  By Theorem 1 in \citet{lv2010} we have that $\boldsymbol\theta^*_\alpha$ is the unique solution to the equation
\begin{equation}\label{pseudoProp}
\mathbf{X}'_\alpha (\boldsymbol\mu - b'(\boldsymbol\theta)) = 0.
\end{equation}
If $g(y)=f_\alpha(y;\boldsymbol\beta_0)$ for some true parameter $\boldsymbol\beta_0$ for any $\alpha$, then $\boldsymbol\beta^* _\alpha= \boldsymbol\beta_0$.  However, if we assume that we are in the non-true model world, then $g(y)$ is not completely specified by any of the candidate models and we refer to $\boldsymbol\beta^* _\alpha$ as the ``pseudo-true parameter'' based on the candidate model $\alpha$.

Similarly to the Gaussian model, for a given $\lambda$, we take $\hat{\boldsymbol\theta}_\lambda = \mathbf{X} \hat{\boldsymbol\beta}_\lambda$ and denote the maximum-likelihood estimator based on the model $\alpha_\lambda$ by $\hat{\boldsymbol\theta}_{\alpha_\lambda} = \mathbf{X}_{\alpha_\lambda} \hat{\boldsymbol\beta}_{\alpha_\lambda}$.  If we let $\hat{\lambda}_n$ denote the regularization parameter selected by a given selection procedure, then the procedure is defined to be \textit{asymptotically loss efficient} if
\[ \frac{L_{KL}(\hat{\boldsymbol\beta}_{\hat{\lambda}_n})}
    {\inf_{\lambda \in [0,\lambda_{max}]} L_{KL}(\hat{\boldsymbol\beta}_n(\lambda))}
        \rightarrow_p 1\]
and $\hat{\boldsymbol\beta}_n(\hat{\lambda}_n)$ is said to be an \textit{asymptotically loss efficient estimator}.

ZLT studied the asymptotic performance of $AIC_\lambda$ in a similar setting.  To establish asymptotic loss efficiency, ZLT restricted the set of candidate models to the set
\[
\mathcal{D} = \{ \alpha :
                    \sup_{\alpha \in \mathcal{D}}
                        |\hat{\boldsymbol\theta}_\alpha -\boldsymbol\theta_0 | \rightarrow 0
                        \text{ in probability, as } n \rightarrow \infty \},
\]
where $\boldsymbol\theta_0 = \mathbf{X}\boldsymbol\beta_0$.  For this restricted set of models, the
maximum-likelihood estimator converges uniformly to the true parameter.  If this set is known in practice, then the model selection process reduces to selecting the most parsimonious model in this set.  This class of models would rarely be known in practice, so this motivates us to weaken this assumption and to prove the efficiency of $AIC_\lambda$ for a general set of candidate models.

Under the regularity conditions (R1)-(R2) given in the supplementary material, \citet{white82} proved that $\hat{\boldsymbol\beta}_\alpha - \boldsymbol\beta^*_\alpha \rightarrow 0$, almost surely, and
established the asymptotic normality of $\hat{\boldsymbol\beta}_\alpha - \boldsymbol\beta^*_\alpha$ under (R1)-(R4).  With the additional condition (R5), \citet{nishii88} applied a Taylor expansion to show that
\begin{equation}\label{estMLE}
\hat{\boldsymbol\beta}_\alpha - \boldsymbol\beta^*_\alpha
    = \mathbf{A}_n^{-1}
    \left\{\frac{1}{n} \frac{\partial l(\boldsymbol\beta^*_\alpha)}{\partial \boldsymbol\beta} + \mathbf{r} \right\}
\end{equation}
for n sufficiently large, where
$
\mathbf{A}_n
    = - \frac{1}{n}
        \partial^2 l(\boldsymbol\beta^*_\alpha)/\partial \boldsymbol\beta \partial \boldsymbol\beta^T
$ and
$ r_{j} = O_p( ||\hat{\boldsymbol\beta}_\alpha - \boldsymbol\beta^*_\alpha||^2)$
for $j=1,\ldots,d_\alpha$.

We define the risk function of the maximum-likelihood estimator to be
$R_{KL}(\hat{\boldsymbol\beta}_\alpha)
    = \E_0(L_{KL}(\hat{\boldsymbol\beta}_\alpha))$.  From Theorem 4 of \citet{lv2010}, under (R1)-(R6),
\[
R_{KL}(\hat{\boldsymbol\beta}_\alpha)
    = L_{KL}(\boldsymbol\beta^*_\alpha)
    + \frac{ tr\{(\mathbf{X}^T_\alpha\mathbf{W}_\alpha\mathbf{X}_\alpha)^{-1}
        \mathbf{X}^T_\alpha\mathbf{W}_{0}\mathbf{X}_\alpha\}}{n}
    + o(1)
\]
where
$\mathbf{W}_{0} = diag\{ \sigma^2_1, \ldots, \sigma^2_n \}$ and
$\mathbf{W}_{\alpha} = diag\{ b''(\theta_{\alpha 1}), \ldots, b''(\theta_{\alpha n}) \}$.  Similarly to the Gaussian model, we further define the random variable
\[
\tilde{R}_{KL}(\hat{\boldsymbol\beta}_{\alpha_\lambda})
    = L_{KL}(\boldsymbol\beta^*_{\alpha_\lambda})
    + \frac{ tr\{(\mathbf{X}^T_{\alpha_\lambda}\mathbf{W}_{\alpha_\lambda}\mathbf{X}_{\alpha_\lambda})^{-1}
        \mathbf{X}^T_{\alpha_\lambda}\mathbf{W}_{0}\mathbf{X}_{\alpha_\lambda}\}}{n}
    + o(1).
\]

With these results and the following assumptions, we can prove the efficiency of $AIC_\lambda$.
\begin{itemize}
\item [(A1$'$)] $( \frac{1}{n} \mathbf{X'}\mathbf{X})^{-1}$ exists and its minimum and maximum eigenvalues are bounded below and above by constant numbers $C_1$ and $C_2$, respectively.
\item[(A2$'$)] $E (y_i-\mu_i)^{4q}<\infty$, for $i = 1, \ldots, n$ and some positive integer $q$.
\item[(A3$'$)] The risks of the maximum-likelihood estimators $\hat{\boldsymbol{\beta}}_\alpha$ satisfy
\[ \sum_{\alpha \in \mathcal{A}_n}
        (nR_{KL}(\hat{\boldsymbol\beta}_\alpha))^{-q} \rightarrow 0.\]
\item[(A4$'$)] $\sup_{\theta} b''(\theta) < \infty$
\item[(A5$'$)] $\sqrt{n} \lambda_{\max} < M_1$ for all $n$ for some constant $M_1>0$.
\item[(A6$'$)] For any $\theta$, $p'(\theta) \leq M_2 \lambda$ for some constant $M_2>0$.
\item[(A7$'$)] $nL_{KL}(\beta^*_{\bar{\alpha}})/d_n \rightarrow \infty$ as $n \rightarrow \infty$.
\end{itemize}
The first three assumptions are analogous to the assumptions made in the Gaussian model, and assumption (A4$'$) is a mild regularity assumption.  As shown by the following lemma, assumptions (A5$'$)-(A7$'$) are sufficient conditions for the penalized estimator to be close to the maximum-likelihood estimator.  These assumptions are analogous to the sufficient conditions used in the Gaussian model.  They are stated explicitly here since they are required in parts of the efficiency proof.

\begin{lemma}\label{l:A4glm}
Under (A5$'$)-(A7$'$),
    \[ \sup_{\lambda \in [0,\lambda_{max}]}
                \frac{ ||\mathbf{b} ||^2}
                    {\tilde{R}_{KL}(\hat{\boldsymbol\beta}_{\alpha_\lambda})}
                        \rightarrow_p 0,\]
where $\mathbf{b}_i$ is a $d_n \times 1$ vector where $b_{i} = p'_\lambda(|\hat{\beta}_{\lambda i})|)sgn(\hat{\beta}_{\lambda i})$ for all $i$ such that $|\hat{\beta}_{\lambda i}|>0$ and is equal to $0$ otherwise.
\end{lemma}

The proof is given in Appendix B.  The next theorem establishes the efficiency of $AIC_\lambda$.

\begin{thm}
Assuming $d_n/n \rightarrow 0$ as $n \rightarrow \infty$, (A1$'$)-(A7$'$) and the regularity conditions (R1)-(R6), the regularization parameter, $\hat{\lambda}_n$, selected by minimizing $AIC_\lambda$ yields an asymptotically loss efficient estimator, $\hat{\beta}_n(\hat{\lambda}_n)$.
\end{thm}
The proof is given in Appendix B.

\section{Simulation Studies}
In this section we study the finite sample performance of the model selection procedures when the true model is not included in the set of candidate models.


In all of the examples, the results are based on 1000 realizations of samples with $n=100, 200,$ and $400$, and the selection procedures are evaluated based on their loss efficiency, loss, and the variability of the selected number of non-zero coefficients. For each realization, if we let $\hat{\lambda}_n$ denote the regularization parameter selected by a given selection procedure, then the loss efficiency is computed as
\begin{equation*}
\frac{L(\hat{\boldsymbol\beta}_{\hat{\lambda}_n})}{\min_{\lambda \in [0,\lambda_{max}]} L(\hat{\boldsymbol\beta}_{\lambda})}.
\end{equation*}
where $L(\cdot)$ is the $L_2$ loss in the linear regression examples and is the KL loss in the GLM examples.  For comparison, we also include results for the (infeasible) ``Optimal'' procedure, which selects the tuning parameter over the regularization path that produces the minimum loss for each realization and report the loss (``Min.Loss'') achieved by this procedure.

\subsection{Linear Regression}
In this section we study the finite sample performance of the model selection procedures discussed in Section 2.2.  The first set of simulations considers a trigonometric regression where the candidate models are in the neighborhood of the true model but never include the true model.  This example is in line with the framework considered by \citet{shibata80} and \citet{hurvich91}.  The second set of simulations look at an example where there is an omitted predictor.  For example, the researcher may have access to some of the relevant predictors but may be missing others.  This is the setting that was considered by ZLT.

\subsubsection{Choice of Penalty Function}
We consider two common choices for the penalty function.  The first is the Smoothly Clipped Absolute Deviation (SCAD) penalty function proposed by \citet{fan01}.  This penalty function is defined by
\begin{equation*}
p'_\lambda(\beta)=\lambda \left \{ I(\beta \leq \lambda) + \frac{(a\lambda-\beta)_+}{(a-1)\lambda} I(\beta > \lambda) \right \}
\end{equation*}
for some $a>2$ and $\beta>0$. \citet{fan01} recommended setting the second tuning parameter in the SCAD penalty function, $a$, equal to 3.7 and this is commonly done in practice; however, doing so will not necessarily guarantee that the SCAD objective function is convex and can result in convergence to local, but non-global, minima.  As a result, in addition to studying the performance of SCAD with $a=3.7$ (SCAD, 3.7), we study the performance of SCAD where $a=\max(3.7,1+1/c^*)$ (SCAD) where $c^*$ is the minimum eigenvalue of $n^{-1}\mathbf{X'X}$.  The latter choice will force the objective function to be convex \citep{Breheny11}.

The wide use of SCAD is mainly due to the fact that it satisfies the ``oracle property.''  This means that, assuming that the true model is in the set of candidate models and subject to certain regularity assumptions, there exists a sequence $\{ \lambda_n\}$ such that if $\lambda_n \rightarrow 0$ and $\sqrt{n}\lambda_n \rightarrow \infty$ then with probability tending to one the SCAD-estimated regression based on the full model will correctly zero out any zero coefficients and have the same asymptotic distribution as the least squares regression based on the correct model.  This result was proven originally for $d_n$ fixed by \citet{fan01} and was extended to the case where $d_n<n$ but $d_n \rightarrow \infty$ by \citet{fan04}.  These results are for an unknown deterministic sequence that needs to be estimated in practice.

The second penalty function that we study is the Lasso proposed by \citet{tibshirani96}.  The Lasso penalty is the $L_1$-norm of the coefficients.  Necessary and sufficient conditions have been established for the Lasso to perform consistent model selection \citep{zhao06}, but in general the Lasso produces biased estimates and does not satisfy the oracle property \citep{zou06}.  However, in the non-true model world, the oracle property has no meaning, since there is no true model.  Further, even in the true model world, the oracle property is an asymptotic property.

It is important to note that although ZLT only studied non-concave penalty functions, if the non-zero estimated coefficients, $\hat{\boldsymbol\beta}_{\lambda 1}$, satisfy a relationship of the form
\begin{equation*}
\hat{\boldsymbol\beta}_{\lambda 1} = (\mathbf{X'}_{\alpha_\lambda}\mathbf{X'}_{\alpha_\lambda})^{-1}\mathbf{X}_{\alpha_\lambda}\mathbf{y}
+ \left(\frac{1}{n}\mathbf{X'}_{\alpha_\lambda}\mathbf{X'}_{\alpha_\lambda}\right)^{-1}\mathbf{b}_1
\end{equation*}
with probability tending to 1 and (A4) is satisfied, then the efficiency proofs will hold for any penalty function.  In the above, $\mathbf{b}_1$ are the elements of $\mathbf{b}$ that correspond to $\hat{\boldsymbol\beta}_{\lambda 1}$.  In particular, based on Lemma 2 of \citet{zou07}, the Lasso satisfies this relationship and the same sufficient conditions provided by ZLT for (A4) can be used.  Therefore, the efficiency proofs will hold for the Lasso, so it is interesting to compare the performance of the two penalty functions.

The Lasso regressions are fit using the R {\tt  lars} package \citep{lars} and the SCAD regressions are fit using the R {\tt ncvreg} package \citep{Breheny11}.  The {\tt lars} package computes the entire regularization path for the Lasso and for SCAD the models are fit over a grid of 200 $\lambda$ values from $\lambda_{min}$ to $\lambda_{max}$, where the first 100 values of $\lambda$ are fit on a log-scale and the last 100 values of $\lambda$ are equally spaced.  \citet{Breheny11} considered a grid of 100 $\lambda$ values in their simulation studies.  We have chosen a grid that is twice as fine in order to remain closer to the theoretical assumption that all possible values of $\lambda$ are considered.  In all simulations, $\lambda_{max}$ is specified so that all of the estimated coefficients are zero and $\lambda_{min}$ is chosen to effectively produce the least squares estimate on the full model.

\subsubsection{Exponential model}
Here we consider a trigonometric example based on an example studied in \citet{hurvich91}.  The true model is the model described as
\begin{equation*}
y_i=e^{4i/n}+\varepsilon_i
\end{equation*}
for $i=1,\ldots,n$, where $\varepsilon_i \iid N(0,\sigma^2)$.  The estimated models are SCAD and Lasso penalized regressions where the matrix of predictors, $\mathbf{X}=(\mathbf{x}^1,\mathbf{x}^2)$, is a $n \times d_n$ matrix with components defined by
\begin{equation*}
x^1_{ij} = \sin \left(\frac{2\pi j}{n} i \right)
\end{equation*}
and,
\begin{equation*}
x^2_{ij} = \cos \left(\frac{2\pi j}{n} i \right)
\end{equation*}
for $j = 1, \ldots, d_n/2$ and $i=1, \ldots, n$.  The maximum number of predictors is allowed to vary by letting the dimension $d_n = 2\lfloor n^c/2 \rfloor$.  It is shown in the appendix that for this example $||\vmu-\mH_{\bar{\alpha}}||^2 \geq k_1 n d_n^{-2}$ for some positive constant $k_1$.  Therefore, by Lemma 2.1, assumption (A3$^*$) will hold so long as $c<1/3$.  In the simulations we take $c=.3$, and for comparison we also consider $c=.5, .8$ and $.98$.  Note that examining $d_n$ close to $n$ allows for the study of high-dimensional data problems, and is in the spirit of simulations performed in \citet{tibshirani96} and \citet{zou05}.  Since the predictor variables are orthogonal in this example, setting $a=3.7$ for SCAD satisfies the convexity constraint for all values of $c$.

As in \citet{hurvich91}, we examine both $\sigma^2=50$ and $\sigma^2=100$, but the patterns for the two error variances are similar so only the results for $\sigma^2=100$ are reported. The median $L_2$ loss efficiency is presented in Table \ref{tab:ExpEff} for both SCAD and Lasso.  For all values of $c$, the median loss efficiency of $AIC_{c_\lambda}$ and $C_{p_\lambda}$ tend to one as the sample size increases, while the median loss efficiency of $BIC_\lambda$ does not show signs of convergence.  These patterns are consistent with the theoretical efficiency results.  When the number of predictor variables is small relative to the sample size, the loss efficiency of $AIC_\lambda$ also tends to one; however, as the number of candidate predictors is increased, the performance of $AIC_\lambda$ deteriorates.  Figure \ref{fig:ExpDf} displays boxplots of the selected number of non-zero coefficients when $n=200$, $\sigma^2=100$, and $c=.98$.  From this plot we see that $AIC_\lambda$ often selects a model that is close to the full model when $c$ is large.  As the sample size is increased the full model becomes less desirable and $AIC_\lambda$ suffers as a result.  For SCAD, $GCV_\lambda$ appears to suffer from a similar problem, but to a lesser extent than $AIC_\lambda$.  The difference in performance for varying values of $c$ suggests that the good asymptotic performance of $AIC_\lambda$ and $GCV_\lambda$ is strongly dependent on the fact that $d_n/n \rightarrow 0$ and these selectors may not perform well in finite samples when this ratio is close to 1.

Overall, the sensitivity to the value of $c$ clearly hurts the performance of $AIC_\lambda$ and can also negatively impact the performance of $C_{p_\lambda}$ and $GCV_{\lambda}$.  The impact on the latter two is more noticeable when looking at SCAD, but in both cases the extreme variability in the size of the selected model is undesirable.  As a result, we recommend the use of $AIC_c$ or 10-fold $CV$, which are less sensitive to the closeness of $d_n$ to $n$.

\begin{table}[H]
\begin{center}
\caption{Median L2 Loss Efficiency over 1000 simulations for the exponential model with $\sigma^2=100$.}
\vskip5pt
\scalebox{0.7}{%
\begin{tabular}{|l|l|cccc|cccc|}	
\hline
	&		&	\multicolumn{8}{c|}{Median Loss Efficiency}									\\															 \hline																						
& & \multicolumn{4}{c|}{SCAD}	&			\multicolumn{4}{c|}{Lasso}			\\
\hline				
Info. Crit.	&	n	&	c=.3	&	c=.5	&	c=.8	&	c=.98	&	c=.3	&	c=.5	&	c=.8	&	c=.98	\\
\hline
10-fold CV	&	100	&	1.00	&	1.05	&	1.07	&	1.08	&	1.00	&	1.01	&	1.05	&	1.12	\\
	&	200	&	1.00	&	1.03	&	1.06	&	1.05	&	1.00	&	1.01	&	1.03	&	1.07	\\
	&	400	&	1.00	&	1.03	&	1.03	&	1.04	&	1.00	&	1.01	&	1.02	&	1.04	\\
\hline																			
$AIC_\lambda$	&	100	&	1.00	&	1.04	&	1.18	&	2.43	&	1.00	&	1.01	&	1.07	&	2.13	\\
	&	200	&	1.01	&	1.02	&	1.20	&	3.08	&	1.00	&	1.01	&	1.06	&	2.57	\\
	&	400	&	1.00	&	1.02	&	1.23	&	4.05	&	1.00	&	1.01	&	1.05	&	3.29	\\
\hline																			
$AIC_{c_\lambda}$	&	100	&	1.00	&	1.04	&	1.09	&	1.13	&	1.00	&	1.02	&	1.10	&	1.21	\\
	&	200	&	1.01	&	1.03	&	1.07	&	1.08	&	1.00	&	1.01	&	1.06	&	1.11	\\
	&	400	&	1.00	&	1.02	&	1.05	&	1.06	&	1.00	&	1.01	&	1.04	&	1.08	\\
\hline																			
$BIC_\lambda$	&	100	&	1.00	&	1.07	&	1.32	&	1.64	&	1.00	&	1.05	&	1.60	&	1.64	\\
	&	200	&	1.02	&	1.06	&	1.47	&	1.51	&	1.00	&	1.06	&	1.74	&	1.62	\\
	&	400	&	1.01	&	1.07	&	1.60	&	1.51	&	1.00	&	1.08	&	1.80	&	1.60	\\
\hline																			
$C_{p_\lambda}$	&	100	&	1.00	&	1.04	&	1.10	&	1.22	&	1.00	&	1.01	&	1.05	&	1.15	\\
	&	200	&	1.01	&	1.02	&	1.09	&	1.15	&	1.00	&	1.01	&	1.03	&	1.09	\\
	&	400	&	1.00	&	1.02	&	1.08	&	1.09	&	1.00	&	1.01	&	1.03	&	1.05	\\
\hline																			
$GCV_\lambda$	&	100	&	1.00	&	1.04	&	1.10	&	1.69	&	1.00	&	1.01	&	1.06	&	1.16	\\
	&	200	&	1.01	&	1.02	&	1.10	&	1.73	&	1.00	&	1.01	&	1.04	&	1.09	\\
	&	400	&	1.00	&	1.02	&	1.08	&	1.82	&	1.00	&	1.01	&	1.03	&	1.05	\\
\hline																												 \end{tabular}}
\end{center}
\label{tab:ExpEff}
\end{table}

\begin{figure}[H]
\centering
\caption{Comparison of model selection procedures based on the number of non-zero coefficients (includes intercept) in the selected model over 1000 simulations for the exponential model with $n=200$, $\sigma^2=100$, and $c=0.98$.}
\subfigure[SCAD]{
\resizebox{6cm}{!}{\includegraphics{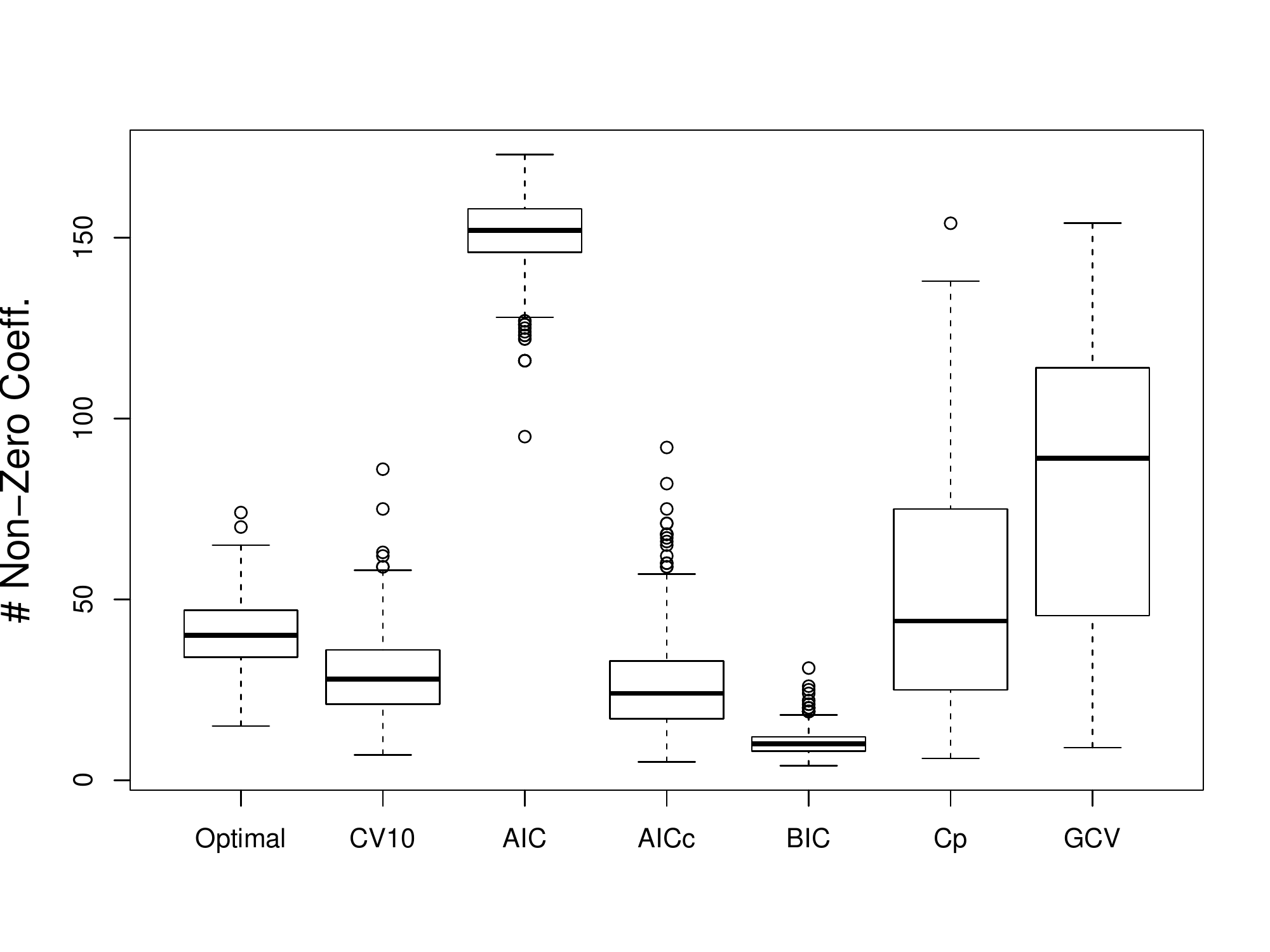}}}
\subfigure[Lasso]{
\resizebox{6cm}{!}{\includegraphics{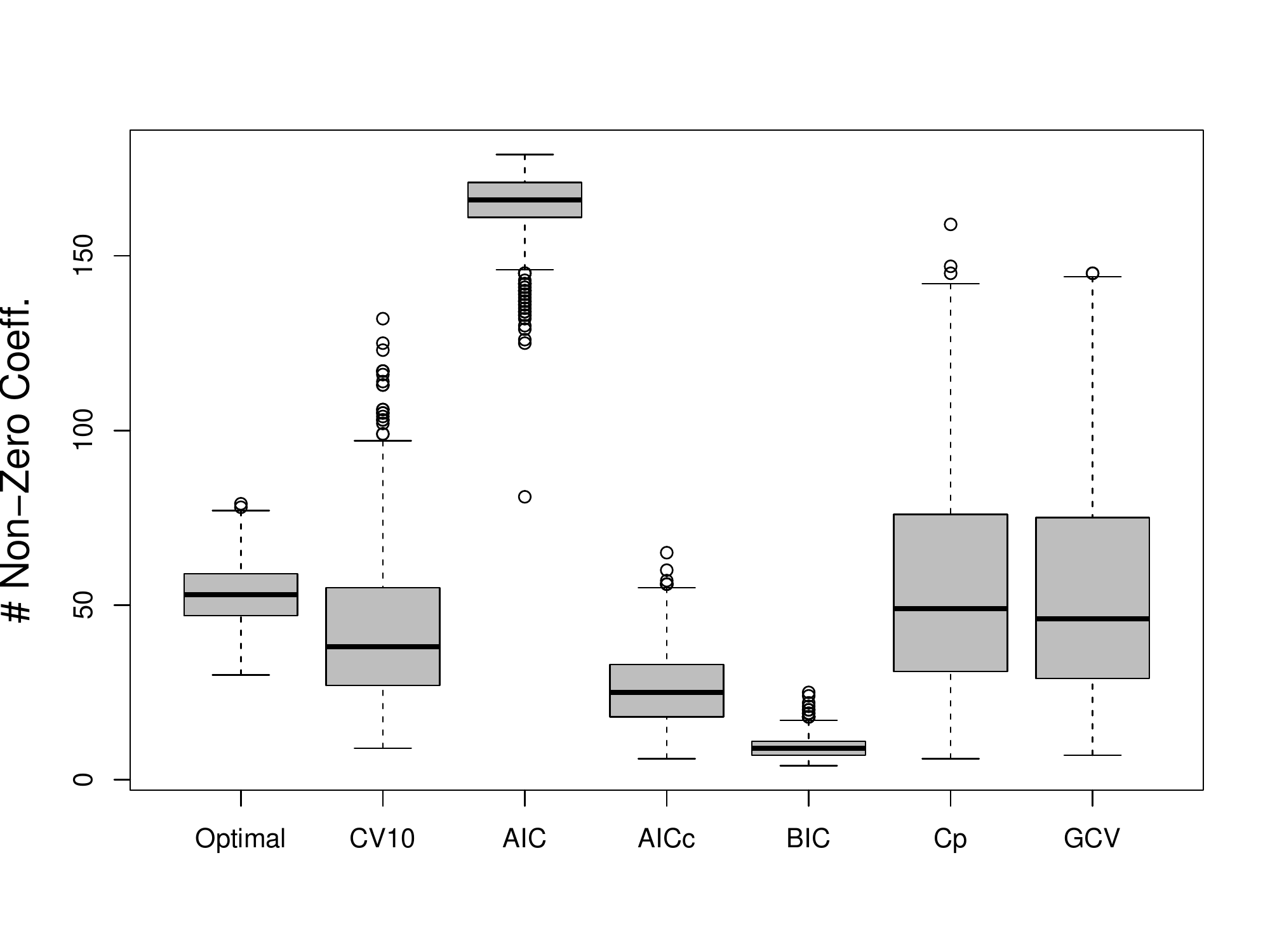}}}
\label{fig:ExpDf}
\end{figure}

Figure \ref{fig:ExpLoss} presents boxplots of the $L_2$ loss for the 1000 realizations when $n=200$ when $c=.5$ and $c=.98$.  From this we can compare the optimal performance of SCAD and the Lasso.  Based on minimum loss, the predictive accuracies of the two methods are similar.  This reinforces that the existence of an oracle property is not relevant in the non-true model world, and an estimator that does not possess the oracle property can still be effective from a predictive point of view.

\begin{figure}[H]
\centering
\caption{Comparison of model selection procedures based on L2 Loss over 1000 simulations for the exponential model with $n=200$ and $\sigma^2=100$.}
\subfigure[c=.5]{
\resizebox{6cm}{!}{\includegraphics{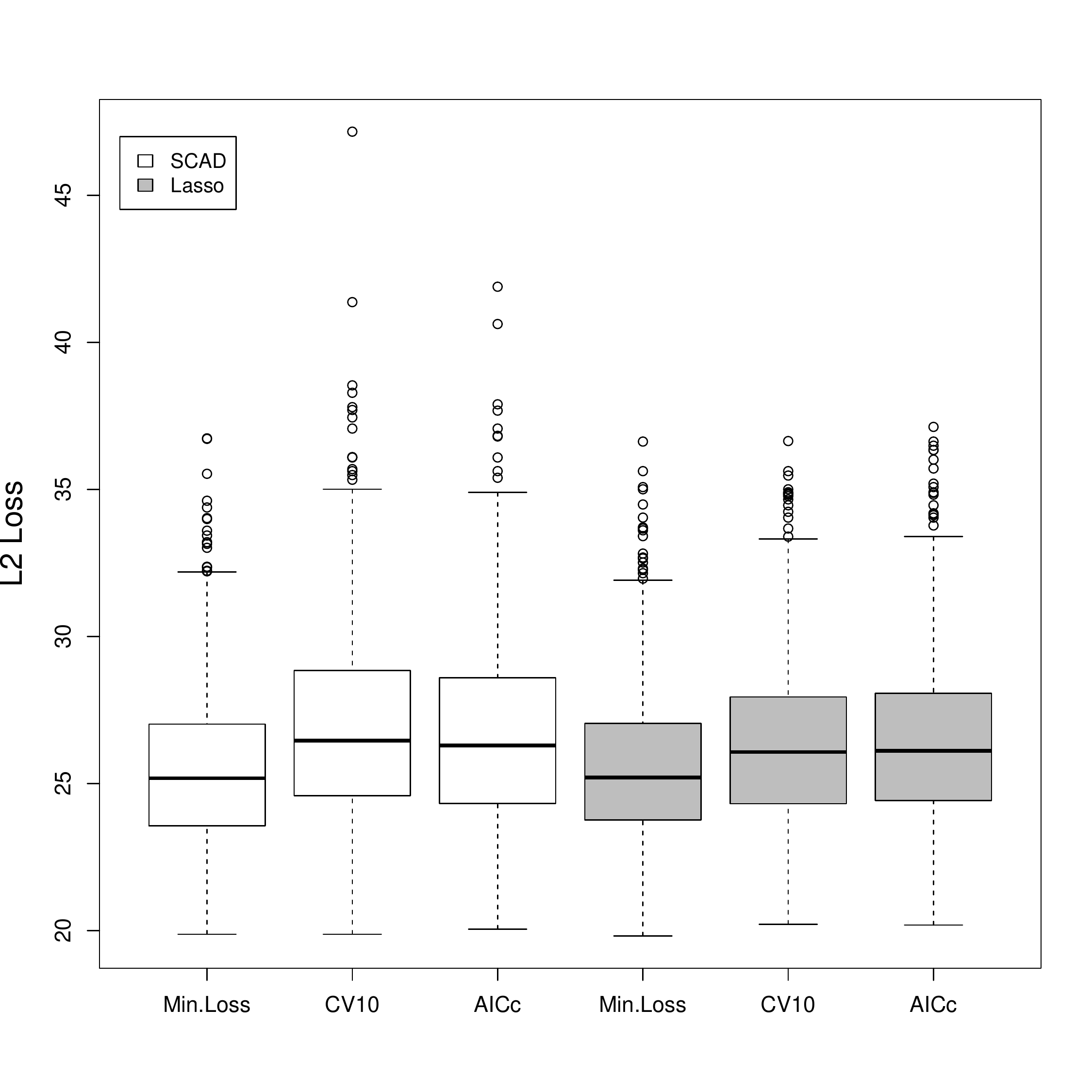}}}
\subfigure[c=.98]{
\resizebox{6cm}{!}{\includegraphics{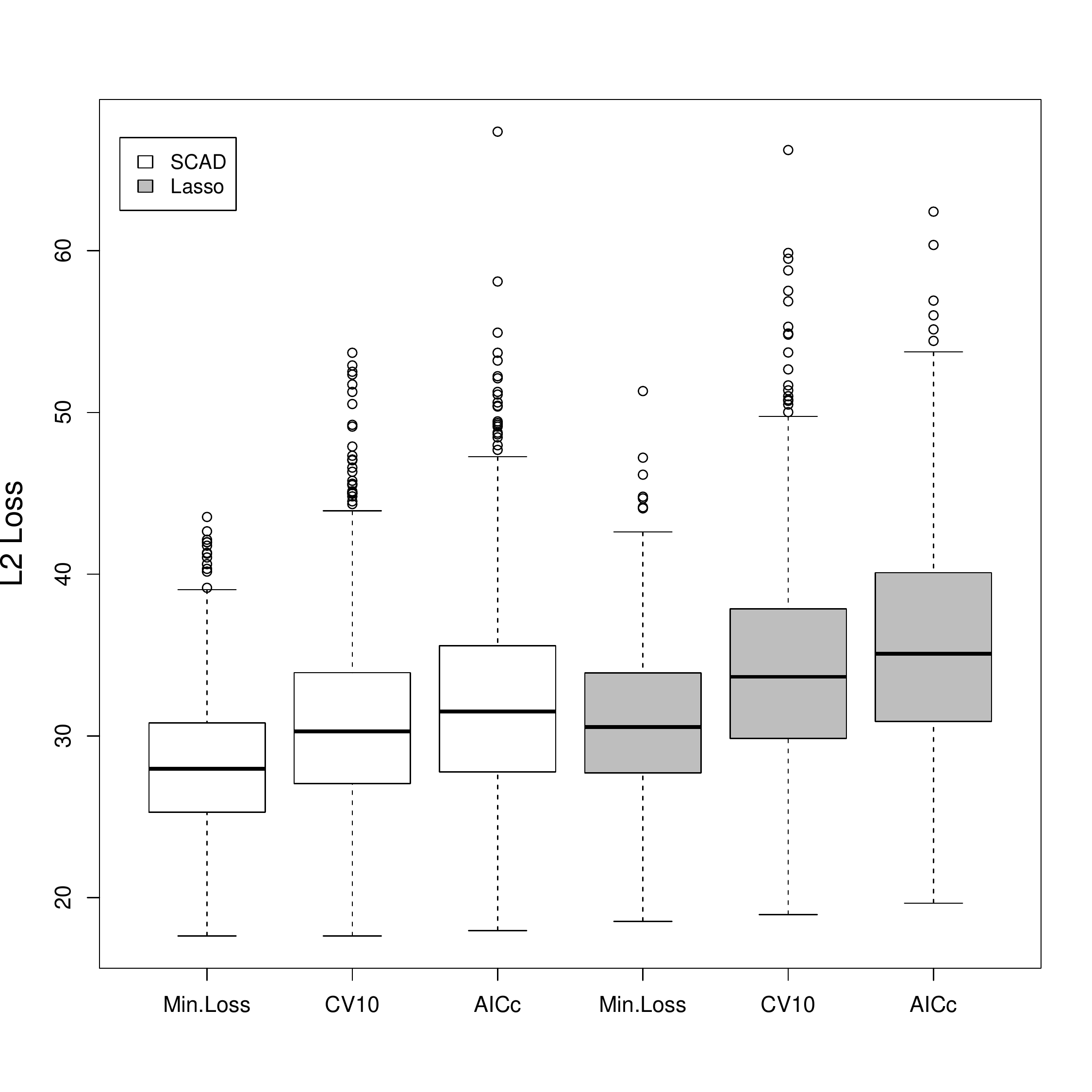}}}
\label{fig:ExpLoss}
\end{figure}

\subsubsection{Omitted Predictor}
Here we study an omitted predictor example similar to example 2 in ZLT.  The true model is defined as
\begin{equation*}
y_i = 3x_{i,1} + 1.5x_{i,2}+2x_{i,10}+x_{i,13}+\varepsilon_i
\end{equation*}
where $\varepsilon_i \iid N(0,\sigma^2)$ for $\sigma^2=16$ and $\sigma^2=25$.  We let $\mathbf{X}$ be a $2n \times (d_n+1)$ matrix of predictors where the $\mathbf{x}_i's$ are simulated from a multivariate normal distribution with mean 0 and variance-covariance matrix $\Sigma$ where $\Sigma_{i,j} = \rho^{|i-j|}$ for $\rho=0$ and $0.5$.  In the simulations $\mX$ is simulated once and is used for every simulation run in order to resemble a fixed $\mX$ setting.  The estimated models are SCAD and Lasso penalized regressions based on the first $n$ observations of $\mathbf{X}$ except with the $13^{th}$ column removed so that the true model is never included in the set of candidate models.  In order to compare predictive performance, we treat the remaining observations of $\mX$ as a hold-out sample and use it to compute the loss for each estimated model.

In both examples the number of superfluous variables included in the candidate models is allowed to vary by letting the dimension $d_n = 2\lfloor n^c/2 \rfloor$.  Under deterministic $\mX$, it is shown in the supplementary material that $||\vmu-\mH_{\bar{\alpha}}||^2 \geq k_1 n$ for some positive constant $k_1$ if the excluded predictor is orthogonal to the included predictors.  By Lemma 2.1, assumption (A3$'$) will then hold if $d_n/n \rightarrow 0$.
This suggests that when the excluded predictor is uncorrelated or only moderately correlated with the included predictors it is reasonable to compare $c=0.5, 0.8$ and $0.98$.

In this example setting $a=3.7$ will not satisfy the convexity constraint for all values of $c$.  Therefore, we further compare the case where $a=3.7$ (SCAD, $a=3.7$) to the case where $a=\max{(3.7,1+1/c^*)}$ (SCAD).

The patterns for the two error variances and two values of $\rho$ are similar so only the results for $\sigma^2=16$ and $\rho=0.5$ are reported.  We first consider Figure \ref{fig:ExclLoss}, which presents boxplots comparing the three estimators based on loss when $n=200$.  From these plots it is immediately clear that all of the information criteria perform better when $a$ is allowed to be data-dependent, while 10-fold $CV$ performs well regardless of the choice of $a$. One possible explanation for this is that all of the information criteria under consideration were derived for use in classical least squares regression so they should perform well assuming that the estimated models are close to the corresponding OLS models.  When the second tuning parameter of SCAD is fixed at 3.7, the objective function is not necessarily convex so the SCAD-estimated models may be very far from the OLS models.  On the other hand, 10-fold CV is a general model selection procedure that should work in a variety of settings.  In general, we recommend using a data-dependent choice of $a$ since it requires little additional cost and can greatly improve the performance of all of the information criteria.


Focusing only on the data-dependent choice of $a$, we see that the performance of the model selection procedures is similar for both SCAD and Lasso when $c=.8$ and when $c=.98$, but that the performance of SCAD is noticeably worse when $c=.5$.  A possible explanation for this is that when $c$ is small, the performance of the SCAD estimators is more sensitive to the choice of the second tuning parameter.  Although taking $a=\max{(3.7,1+1/c^*)}$ guarantees that the penalized loss function is convex, it may not be the optimal choice for this parameter and more investigation into the choice of this parameter is needed.  Of course, this implies an advantage of Lasso over SCAD, since it does not require the choice of this second parameter.

Comparing the model selection procedures, we again see that $AIC_\lambda$, $GCV_\lambda$, and $C_{p_\lambda}$ are sensitive to the number of predictor variables while $AIC_{c_\lambda}$ and 10-fold $CV$ maintain good performance.  The boxplots of the selected number of non-zero coefficients are omitted since the patterns are similar to those seen in the exponential model.  In Figure \ref{fig:ExclLoss} it is clear that this sensitivity to the value of $c$ impacts the performance of the model selection procedures, and as a result 10-fold $CV$ and $AIC_{c_\lambda}$ outperform the other procedures.  10-fold $CV$ outperforms $AIC_{c_\lambda}$ in some scenarios, but, in general, the performance of the two methods appears to be comparable.

\begin{figure}[H]
\centering
\caption{Comparison of model selection procedures based on L2 Loss on new design points over 1000 simulations for the model with an omitted predictor with $n=200$ and $\rho=0.5$. In order to make it easier to compare the procedures, the limits of the vertical axis are specified so that all the boxes and whiskers appear but some of the outliers are not shown.}
\subfigure[c=.5]{
\resizebox{16cm}{!}{\includegraphics{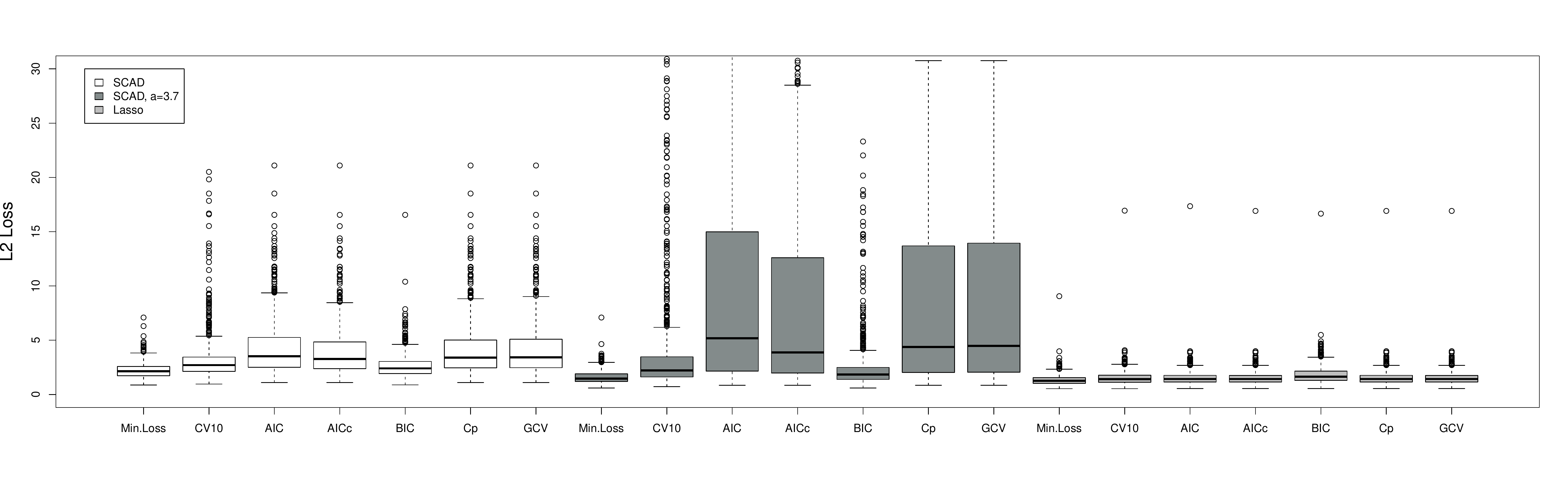}}}
\subfigure[c=.8]{
\resizebox{16cm}{!}{\includegraphics{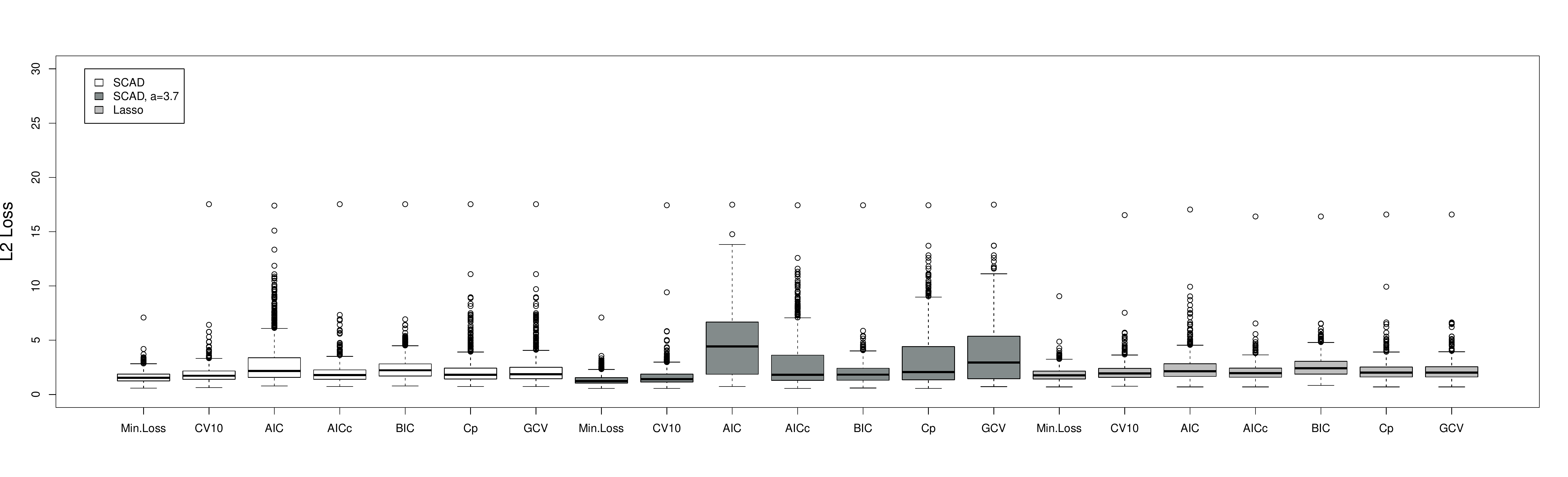}}}
\subfigure[c=.98]{
\resizebox{16cm}{!}{\includegraphics{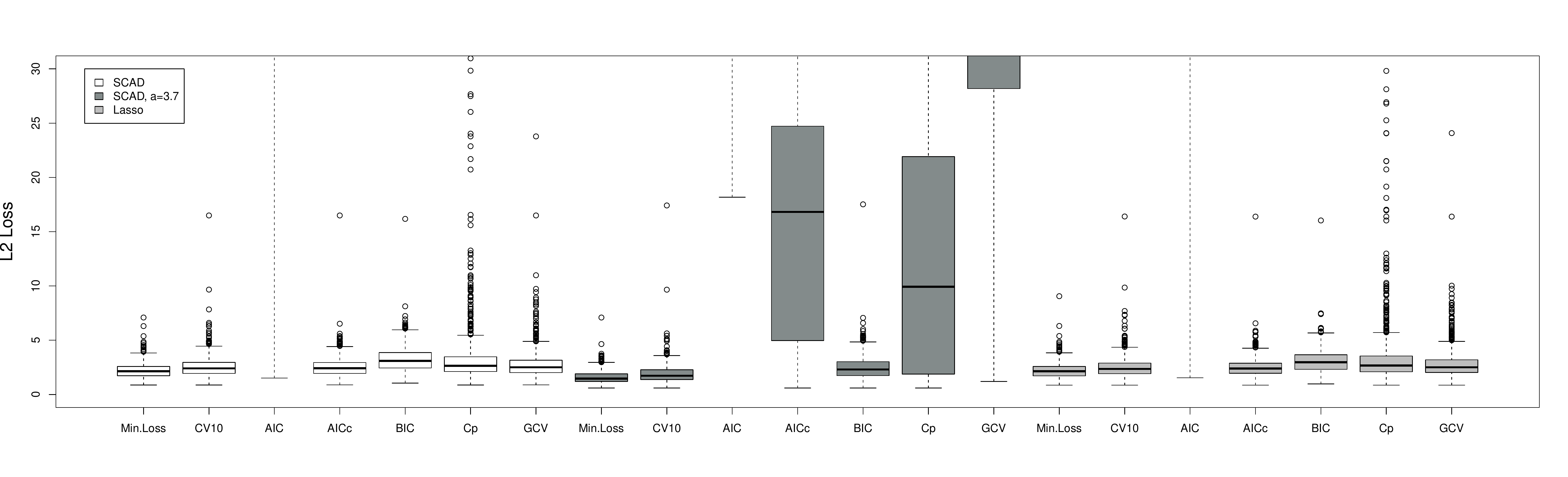}}}
\label{fig:ExclLoss}
\end{figure}

In order to study the asymptotic behavior of the selection procedures, Table 2 presents the median loss efficiencies.  With the exception of SCAD with $c=0.5$, the loss efficiencies of $AIC_{c_\lambda}$, $C_{p_\lambda}$, and $GCV_\lambda$ tend to one, while the loss efficiency of $BIC_\lambda$ does not show signs of convergence.  Also, the results again show that $AIC_\lambda$ performs poorly when the number of predictor variables is large relative to the sample size.  For SCAD with $c=0.5$, the loss efficiency of the efficient methods do not show signs of converging to one, which further suggests that the second tuning parameter may not be optimally selected.  Overall, the results corroborate the theoretical findings, but reinforce that the finite sample performance of asymptotically equivalent methods may vary greatly.

\begin{table}[H]
\begin{center}
\caption{Median L2 Loss Efficiency on new design points over 1000 simulations for the model with an omitted predictor with $\rho=0.5$.}
\vskip5pt
\scalebox{0.7}{%
\begin{tabular}{|l|l|ccc|ccc|ccc|}
\hline
	&		&	\multicolumn{9}{c|}{Median Loss Efficiency}					\\
\hline
&	&	\multicolumn{3}{c|}{SCAD}&	\multicolumn{3}{c|}{SCAD, a=3.7} &	 \multicolumn{3}{c|}{Lasso}					\\
\hline
Info. Crit.	&	n	&	c=.5	&	c=.8	&	c=.98	&	c=.5	&	c=.8	&	c=.98	&	c=.5	&	c=.8	&	c=.98\\
\hline
10-fold CV	&	100	&	1.32	&	1.10	&	1.09	&	1.72	&	1.23	&	1.20	&	1.08	&	1.09	&	1.08	 \\
	&	200	&	1.19	&	1.07	&	1.07	&	1.35	&	1.08	&	1.10	&	1.05	&	1.06	&	1.05	\\
	&	400	&	1.14	&	1.05	&	1.05	&	1.26	&	1.02	&	1.04	&	1.04	&	1.04	&	1.04	\\
\hline																					
$AIC_\lambda$	&	100	&	1.49	&	1.44	&	41.44	&	2.78	&	2.57	&	37.24	&	1.08	&	1.19	&	 37.64	\\
	&	200	&	1.57	&	1.24	&	51.80	&	3.03	&	3.30	&	59.94	&	1.06	&	1.12	&	49.77	\\
	&	400	&	1.84	&	1.11	&	67.73	&	4.13	&	3.16	&	76.07	&	1.04	&	1.07	&	64.94	\\
\hline																					
$AIC_{c_\lambda}$	&	100	&	1.36	&	1.13	&	1.10	&	2.19	&	1.45	&	4.27	&	1.07	&	1.09	&	 1.08	\\
	&	200	&	1.41	&	1.08	&	1.07	&	2.45	&	1.27	&	10.10	&	1.06	&	1.07	&	1.06	\\
	&	400	&	1.68	&	1.06	&	1.05	&	3.31	&	1.10	&	17.28	&	1.04	&	1.04	&	1.05	\\
\hline																					
$BIC_\lambda$	&	100	&	1.12	&	1.26	&	1.40	&	1.26	&	1.41	&	1.62	&	1.11	&	1.24	&	1.31	 \\
	&	200	&	1.07	&	1.39	&	1.40	&	1.13	&	1.31	&	1.38	&	1.21	&	1.34	&	1.33	\\
	&	400	&	1.05	&	1.31	&	1.32	&	1.07	&	1.16	&	1.25	&	1.21	&	1.28	&	1.27	\\
\hline																					
$C_{p_\lambda}$	&	100	&	1.42	&	1.17	&	1.22	&	2.40	&	1.65	&	3.04	&	1.08	&	1.12	&	1.24	 \\
	&	200	&	1.46	&	1.10	&	1.16	&	2.69	&	1.42	&	5.27	&	1.06	&	1.08	&	1.14	\\
	&	400	&	1.75	&	1.07	&	1.08	&	3.75	&	1.14	&	10.93	&	1.04	&	1.05	&	1.08	\\
\hline																					
$GCV_\lambda$	&	100	&	1.43	&	1.20	&	1.13	&	2.49	&	2.01	&	14.24	&	1.07	&	1.11	&	1.12	 \\
	&	200	&	1.48	&	1.11	&	1.10	&	2.75	&	2.10	&	25.71	&	1.06	&	1.08	&	1.09	\\
	&	400	&	1.78	&	1.07	&	1.06	&	3.86	&	1.26	&	35.72	&	1.04	&	1.05	&	1.06	\\
\hline																					
\end{tabular}}
\end{center}
\label{tab:ExclEff}
\end{table}

\subsection{Poisson Regression}
In this section we present simulation results for GLMs with no dispersion parameter.  For GLMs, it is less clear how to handle the second tuning parameter for SCAD. \citet{Breheny11} recommended using an adaptive rescaling technique, but it is unclear how such a procedure will impact the performance of the model selection procedures and initial simulations for Bernoulli data resulted in convergence issues.  As a result we only study the Lasso in this section.  The {\tt lars} package is only designed for linear regression, so we instead work with the R {\tt glmpath} package \citep{park11}, which fits the entire regularization path for the Lasso for GLMs.

We consider a trigonometric example based on an example studied in \citet{hurvich91}.  We take $\theta_t = e^{-5i/n}$ for $t=0,\ldots,n-1$ and simulate $y_t$ from a Poisson distribution with $\mu_t = \exp(\theta_t)$.  The estimated models are Lasso penalized Poisson regressions where the matrix of predictors, $\mathbf{X}=(\mathbf{x}^1,\mathbf{x}^2)$, is a $n \times d_n$ matrix with components defined as in the exponential
model.  Similar to before, we vary the maximum number of predictors by letting the dimension $d_n = 2\lfloor n^c/2 \rfloor$
and we compare $c=.3$, $c=.5$ and $c=.8$.  The case with $c=.98$ is omitted due to convergence problems with the package.

Although $AIC_{c}$ was originally derived for linear regression, its use is commonly recommended in a more general setting when the number of predictor variables is large relative to the sample size (\citet{burnham02}, p. 66).  We therefore compare the performance of $AIC_\lambda$ to 10-fold $CV$, $AIC_{c_\lambda}$ and $BIC_\lambda$ where
\[
AIC_{c_\lambda} = -\frac{2}{n}l(\hat{\boldsymbol\beta_\lambda}) + 2 \frac{df_\lambda + 1}{n-df_\lambda-2}
\]
and
\[
BIC_{\lambda} = -\frac{2}{n}l(\hat{\boldsymbol\beta_\lambda}) + \log(n) \frac{df_\lambda}{n}.
\]
Table \ref{tab:GLMEff} reports the median KL loss efficiencies over the 1000 simulations.  In all three cases, $AIC_\lambda$, $AIC_{c_\lambda}$, and 10-fold $CV$ show signs of converging to one and have comparable performance, whereas $BIC_\lambda$ performs noticeably worse and does not show signs of convergence.  Figure \ref{fig:GLMdf} presents boxplots of the selected number of non-zero coefficients.  This figure suggests that the poor performance of $BIC_\lambda$ is due to its tendency to select models that are too sparse.  In comparison, the other procedures select models with dimension closer to the optimal dimension.  Overall, these results are consistent with the theoretical findings.

\begin{table}[H]
\begin{center}
\caption{Median KL Loss Efficiency over 1000 simulations for the poisson model.}
\vskip5pt
\scalebox{0.7}{%
\begin{tabular}{|l|l|ccc|}
\hline
	&		&	\multicolumn{3}{c|}{Median Loss Efficiency}																	\\
\hline
	&		&	\multicolumn{3}{c|}{Lasso} \\
\hline																								
Info. Crit.	&	n	&	c=.3	&	c=.5	&	c=.8	\\
\hline
10-fold CV	&	100	&	1.08	&	1.30	&	1.17	\\
	&	200	&	1.00	&	1.19	&	1.17	\\
	&	400	&	1.00	&	1.08	&	1.09	\\
\hline									
$AIC_\lambda$	&	100	&	1.01	&	1.19	&	1.15	\\
	&	200	&	1.01	&	1.13	&	1.10	\\
	&	400	&	1.01	&	1.08	&	1.08	\\
\hline									
$AIC_{c_\lambda}$	&	100	&	1.02	&	1.18	&	1.10	\\
	&	200	&	1.01	&	1.13	&	1.07	\\
	&	400	&	1.01	&	1.08	&	1.06	\\
\hline									
$BIC_\lambda$	&	100	&	1.38	&	1.38	&	1.14	\\
	&	200	&	1.48	&	1.63	&	1.27	\\
	&	400	&	1.30	&	1.85	&	1.45	\\
\hline
\end{tabular}}
\end{center}
\label{tab:GLMEff}
\end{table}

\begin{figure}[H]
\centering
\caption{Comparison of model selection procedures based on the number of non-zero coefficients (includes intercept) in the selected model over 1000 simulations for the poisson model with $n=200$}
\subfigure[c=.3]{
\resizebox{6cm}{!}{\includegraphics{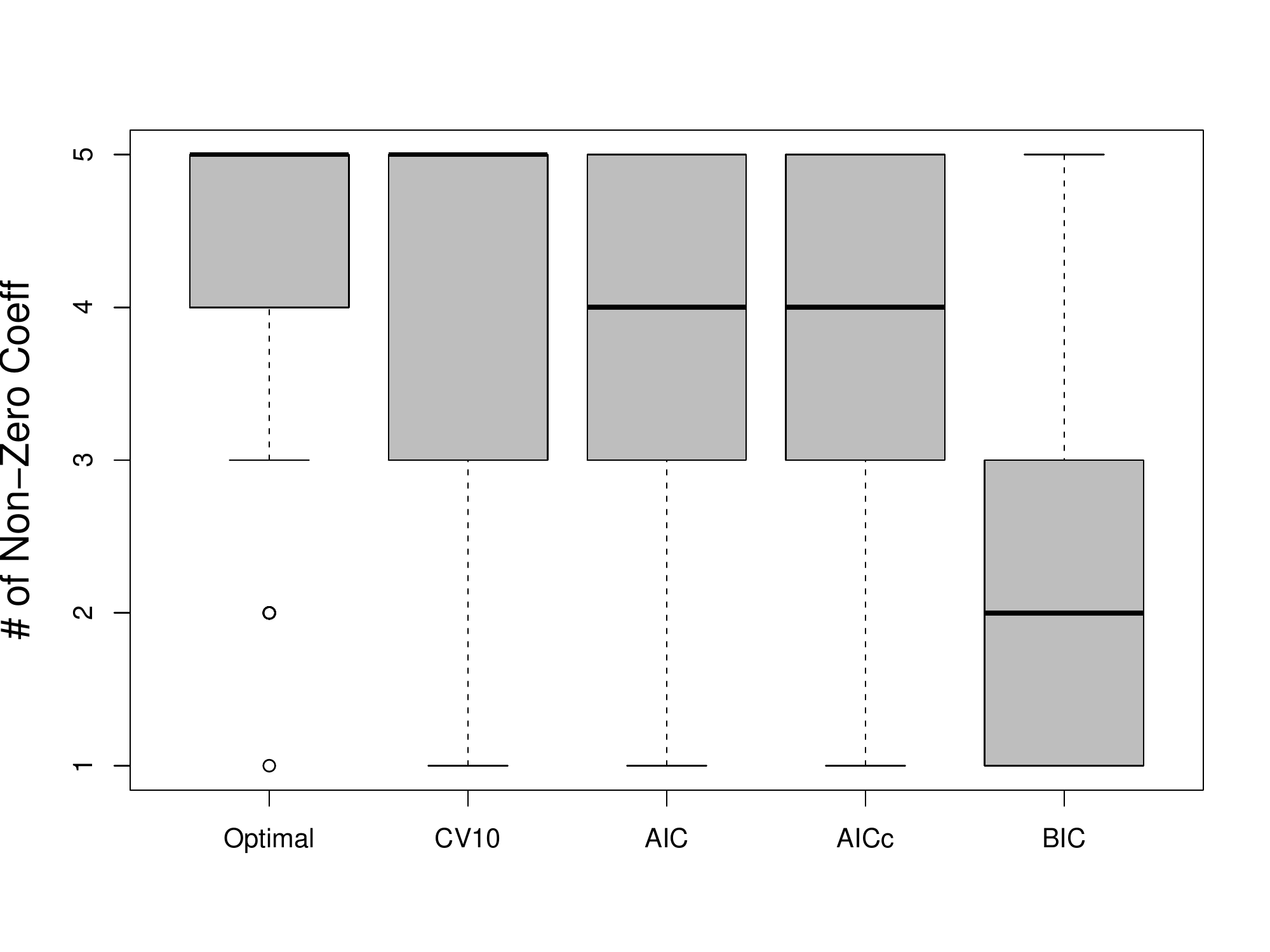}}}
\subfigure[c=.5]{
\resizebox{6cm}{!}{\includegraphics{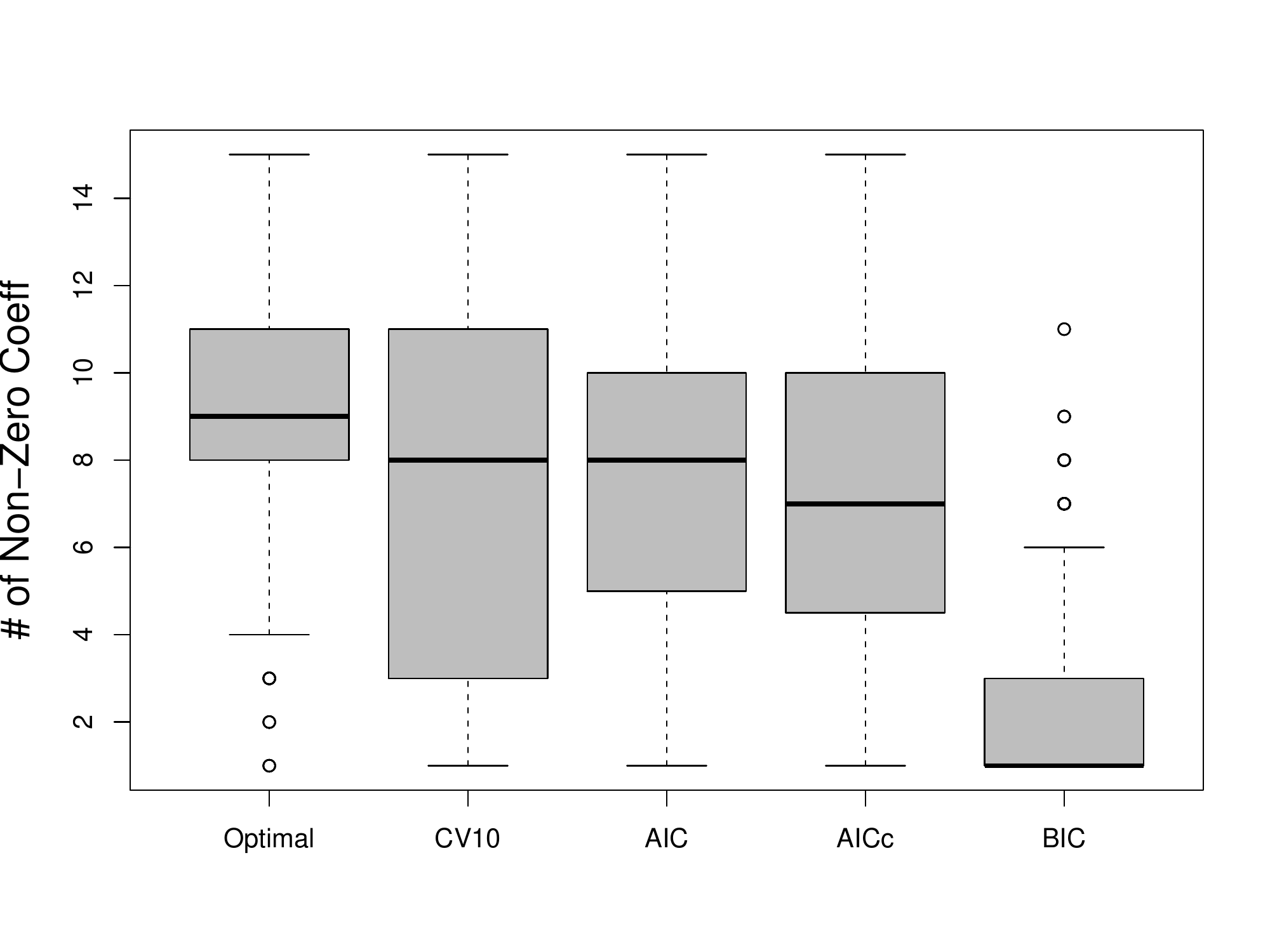}}}
\subfigure[c=.8]{
\resizebox{6cm}{!}{\includegraphics{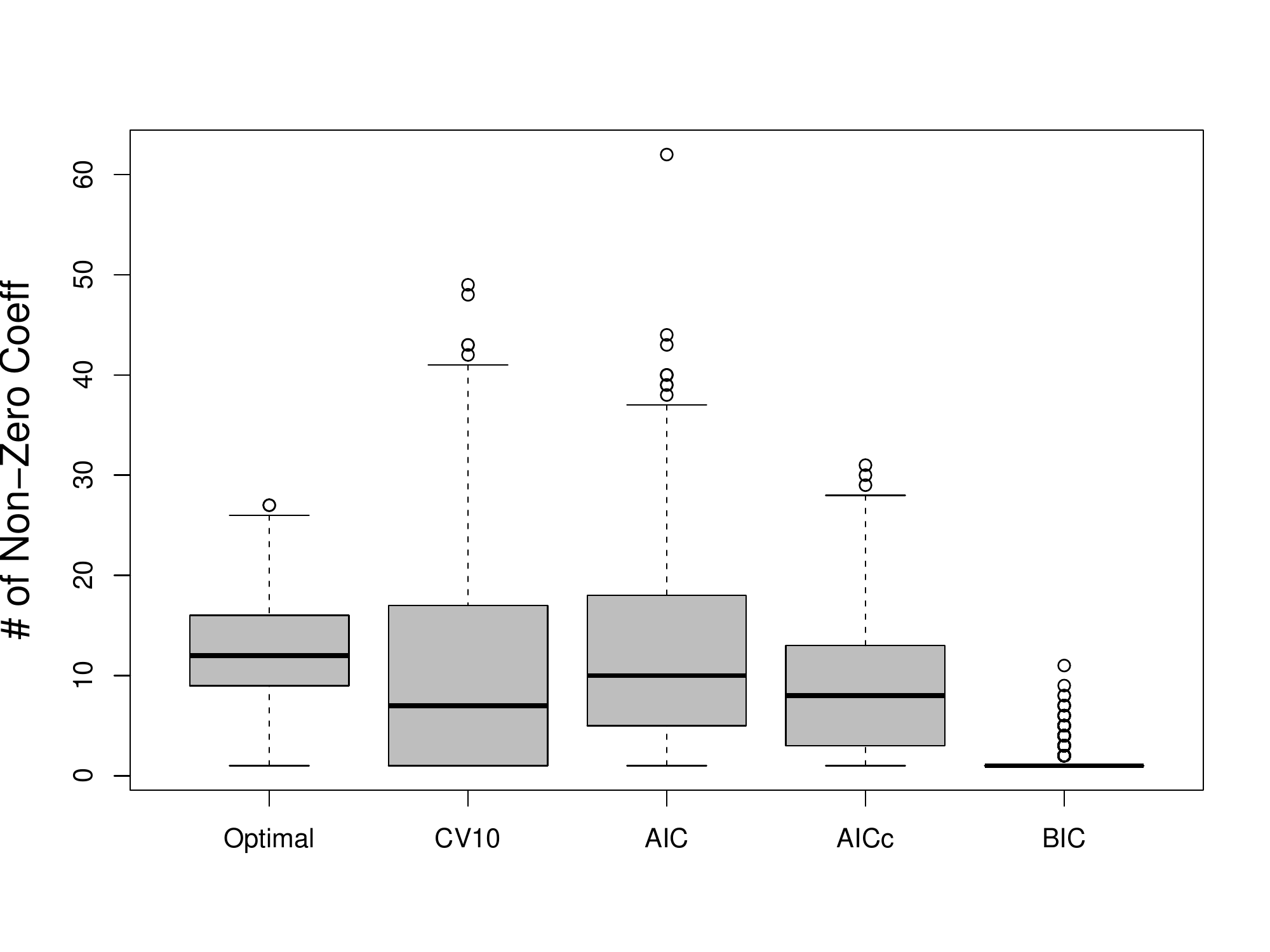}}}
\label{fig:GLMdf}
\end{figure}

\section{Analysis of a Real Data Set}

We now consider the transaction data set from \citet{sela12} in order to compare the candidate models chosen by the regularization parameter selectors when applied to a real world data set.  The data contains transactions for third-party sellers on Amazon
Web Services and the goal is to predict the prices at which software titles are sold
based on the characteristics of the competing sellers. The target variable is the price premium that a seller can command (the difference between the price at which the
good is sold and the average price of all of the competing goods in
the marketplace).  There are 24 potential predictors which include the seller's reputation
(the total number of comments and the number of positive and
negative comments received from buyers over different time periods),
the length of time that the seller has been in the marketplace, the
number of competitors, the quality of competing goods in the
marketplace, the average reputation of the competitors, and the
average prices of the competing goods.  The data set contains 100 observations.

Table \ref{dataAnalysis} reports the results for the information criteria as well as 10-fold $CV$ based on two different runs (and hence two different random divisions of the data), which are referred to as 10-fold $CV$ (1) and 10-fold $CV$ (2).  Only six predictor variables were ever selected so the remaining variables are omitted from the table.  It is clear that the variables selected are heavily reliant on the selection procedure and the penalty function chosen.  In particular, there is a noticeable difference between the variables selected by $AIC_\lambda$ and $AIC_{c_\lambda}$, and in all three cases $BIC_\lambda$ selected a model with no predictors, suggesting that it may be selecting an underfitted model.  If we approach this problem from a predictive point of view, we know that there is little advantage to using SCAD over the Lasso, but that the choice of the second tuning parameter can greatly impact the performance of the former.  Therefore, we recommend focusing on the Lasso.  From the simulations we know that 10-fold $CV$ maintains good performance in a variety of settings.  However, it is 10 times more expensive to implement than using an information criterion, the asymptotic properties of 10-fold $CV$ are not fully understood in this context, and the randomness involved in the procedure makes it difficult for data analysts to reproduce results.  In the case of the Lasso, this last point is reinforced by the change in the selected variables between the two runs of 10-fold $CV$, as in the first run four nonzero coefficients were estimated, while in the second run none were.  We recommend proceeding using $AIC_{c_\lambda}$ as the selector of the tuning parameter for the Lasso as an alternative that avoids these issues.

\begin{table}[H]\label{dataAnalysis}
\begin{center}
\caption{Selected variables for transaction data.}
\vskip5pt
\scalebox{0.5}{%
\begin{tabular}{|l|cccccc|}
\hline
    & Ave. Comp. &  Ave. Comp. & Ave. Comp. &  Seller &  Negative & Negative \\
    &  Price &  Condition &  Rating &  Condition &  Comments &  Comments \\
\multicolumn{1}{|c|}{Selector}  &  &  &  &  & (30 days) & (Lifetime) \\
\hline
\multicolumn{7}{|c|}{SCAD} \\
\hline
10-fold CV (1)  &   X &  &  &  &  &  \\
10-fold CV (2)  &  X  &  &  &  &  &  \\
$AIC_\lambda$  &   X & X & X & X & X &  \\
$AIC_{c_\lambda}$  &   X &  &  &  &  &  \\
$BIC_\lambda$  &    &  &  &  &  &  \\
$C_{p_\lambda}$  &   X &  &  &  &  &  \\
$GCV_\lambda$  &   X & X & X & X & X &  \\
\hline
\multicolumn{7}{c}{} \\
\hline
\multicolumn{7}{|c|}{SCAD ($a=3.7$)} \\
\hline
10-fold CV (1)  &   X &  &  &  &  &  \\
10-fold CV (2)  &  X  &  &  &  &  &  \\
$AIC_\lambda$  &   X & X & X & X & X & X  \\
$AIC_{c_\lambda}$  &  X & X & X & X & X & X  \\
$BIC_\lambda$  &    &  &  &  &  &  \\
$C_{p_\lambda}$  &   X & X & X & X & X & X \\
$GCV_\lambda$  &   X & X & X & X & X & X  \\
\hline
\multicolumn{7}{c}{} \\
\hline
\multicolumn{7}{|c|}{LASSO} \\
\hline
10-fold CV (1)  &   X & X &  & X & X &  \\
10-fold CV (2)  &    &  &  &  &  &  \\
$AIC_\lambda$  &   X & X & X & X & X &  \\
$AIC_{c_\lambda}$  &   X &  &  &  &  &  \\
$BIC_\lambda$  &    &  &  &  &  &   \\
$C_{p_\lambda}$  &  X  &  &  &  &  &  \\
$GCV_\lambda$  &  X  &  &  &  &  &  \\
\hline
\end{tabular}}
\end{center}
\end{table}

\section{Concluding Remarks}

This paper studied the asymptotic and finite sample performance of classical model selection procedures in the context of penalized likelihood estimators without the assumption that the true model is included amongst the candidate models.  We proved that $AIC_\lambda$, $AIC_{c_\lambda}$, $C_{p_\lambda}$, and $GCV_\lambda$ are efficient selectors of the regularization parameter for regularized regression, and the numerical studies for regularized regression yielded several interesting observations.  As anticipated, we found that $BIC_\lambda$ is outperformed by the efficient model selection procedures and demonstrated that $AIC_{\lambda}$, $BIC_\lambda$, $C_{p_\lambda}$, and $GCV_\lambda$ are all sensitive to the number of predictor variables that are included in the full model and that their performance can suffer as a result.  In light of this issue we recommend that researchers use a method that is insensitive to the number of variables included in the model.  From the simulations, 10-fold $CV$ has the best overall performance.  However, the discussion in Section 5 noted some of the disadvantages of this method including computational cost and variable results due to the inherent randomness of the procedure.  As an alternative, data analysts can consider using $AIC_{c_\lambda}$, which was shown here to be an efficient selection procedure for the tuning parameter, and which the simulations suggest has comparable performance to that of 10-fold $CV$.  Lastly, the simulations suggest that there is no clear advantage to using SCAD in a world where the ``oracle property'' does not apply.  Combining this with the facts that the Lasso can be fitted using the efficient `Lars' algorithm and does not involve a second tuning parameter that can greatly impact results, researchers may prefer to use the Lasso if they feel that they are in the non-true model world.

 To further generalize our results, we also proved that $AIC_\lambda$ is an efficient selector of the regularization parameter for regularized GLMs with no dispersion parameter and used numerical studies to compare its performance to that of $AIC_{c_\lambda}$, $BIC_\lambda$ and 10-fold $CV$.  Again, the performance of $BIC_\lambda$ was noticeably worse than the other procedures, and the performances of $AIC_\lambda$, $AIC_{c_\lambda}$ and 10-fold $CV$ were comparable to each other, supporting our recommendation for the use of $AIC_{c_\lambda}$.  Extending these results to GLMs with an unknown dispersion parameter is an
 interesting open problem.  In this setting it is necessary to work with extended quasi-likelihood methods.  Although model selection criteria such as $AIC_c$ have been proposed in such settings as \citep{hurvich95}, the extended quasi-likelihood is not a true likelihood so the results of \citet{white82} and \citet{nishii88} do not apply.  Investigations into the properties of model selection procedures in this setting is an area for future research.

As a final remark, this paper dealt with the case when $d_n/n \rightarrow 0$, and the theoretical results cannot be directly extended to the case when $d_n/n$ converges to something other than zero.  The latter setting has received a great deal of attention in recent literature (in particular $d_n \gg n$) and is an area for future investigation.

\appendix

\newcommand{\appsection}[1]{\let\oldthesection\thesection
  \renewcommand{\thesection}{Appendix \oldthesection}
  \section{#1}\let\thesection\oldthesection}

\appsection{}

\begin{proof}[Proof of Lemma \ref{A3_lemma}]
By definition
\[
n R(\hat\vbeta_{\alpha}) \geq || \vmu - \mH_{\alpha} \vmu||^2\geq || \vmu - \mH_{\bar\alpha} \vmu||^2.
\]
Then by (\ref{bias_bound}) and (\ref{A3_cond}),
\[
\sum_{\alpha=1}^{2^{d_n}}
    (n R(\hat\beta_\alpha)) ^{-q}
        \leq
            2^{d_n} k_1^{-q}n^{-q}d_n^{-qk_2}
            = 2^{ log_2(n) \left( \frac{d_n}{\log_2(n)} - \frac{qlog_2(k_1)}{log_2(n)} - \frac{qk_2\log_2(d_n)}{log_2(n)} \right)}
            \to 0.
\]
Next by (\ref{bias_bound}) and (\ref{A3'_cond}),
\[
\sum_{\alpha=1}^{2^{d_n}}
\delta^{n R(\hat\beta_\alpha)}
    \leq
        2^{d_n}
        \delta^{k_1 n^{k_2}}
    = 2^{ d_n \left( 1 + \frac{k_1 n d_n^{k_2}log_2(\delta)}{d_n} \right) }
    \to 0.
\]
\end{proof}

Before proving Theorems 1, 2, and 3, we establish the following two lemmas.
\begin{lemma}
Assume that (A1)-(A4) hold and that $d_n/n \rightarrow 0$ as $n \rightarrow \infty$. Then
\[ \sup_{\lambda \in [0,\lambda_{max}]} \frac{ d_{\alpha_\lambda}|\hat{\sigma}^2_\lambda - \sigma^2| }{nL(\hat{\beta}_\lambda)}
    \rightarrow_p 0.\]
\end{lemma}
\begin{proof}
The technique used to prove this result is similar to the proof of Theorem 2 in \citet{shibata81}.  First consider
{\small
\begin{align*}
 |\hat{\sigma}^2_\lambda - \sigma^2| & = \left | \frac{||\mathbf{y}-\hat{\boldsymbol\mu}_\lambda||^2}{n} - \sigma^2 \right | \\
& \leq \left | \frac{||\mathbf{y}-\hat{\boldsymbol\mu}_{\alpha_\lambda}||^2}{n} - \sigma^2 \right | + \frac{||\hat{\boldsymbol\mu}_{\alpha_\lambda}-\hat{\boldsymbol\mu}_{\lambda}||^2}{n} \\
& \leq L(\hat{\boldsymbol\beta}_{\alpha_\lambda}) + 2\left |\frac{ \boldsymbol{\varepsilon}^T (\boldsymbol\mu - \hat{\boldsymbol\mu}_{\alpha_\lambda})}{n}\right |+
\left | \frac{||\boldsymbol\varepsilon||^2}{n} - \sigma^2 \right | + \frac{||\hat{\boldsymbol\mu}_{\alpha_\lambda}-\hat{\boldsymbol\mu}_{\lambda}||^2}{n}.
\end{align*}}
Applying the Cauchy-Schwarz inequality, it follows that
{\small
\begin{align*}
 |\hat{\sigma}^2_\lambda - \sigma^2|  &  \leq L(\hat{\boldsymbol\beta}_{\alpha_\lambda}) + 2 ||\boldsymbol\varepsilon|| \frac{||\boldsymbol\mu - \hat{\boldsymbol\mu}_{\alpha_\lambda}||}{n}+
\left | \frac{||\boldsymbol\varepsilon||^2}{n} - \sigma^2 \right | + \frac{||\hat{\boldsymbol\mu}_{\alpha_\lambda}-\hat{\boldsymbol\mu}_\lambda||^2}{n}.
\end{align*}}
Then
{\small
\begin{align*}
\frac{|\hat{\sigma}^2_\lambda - \sigma^2|d_{\alpha_\lambda}}{nL(\hat{\boldsymbol\beta}_\lambda)}
& \leq
\frac{d_{\alpha_\lambda}}{n} \left [\frac{L(\hat{\boldsymbol\beta}_{\alpha_\lambda})}{L(\hat{\boldsymbol\beta}_\lambda)} \right] \\
& + \frac{2}{\sigma} \left [ \frac{d_{\alpha_\lambda}}{n} \frac{||\boldsymbol\varepsilon||^2}{n} \right ]^{1/2}
\left [ \frac{L(\hat{\boldsymbol\beta}_{\alpha_\lambda})}{L(\hat{\boldsymbol\beta}_\lambda)} \right ]^{1/2}
\left [ \frac{\sigma^2 d_{\alpha_\lambda}}{n\tilde{R}(\hat{\boldsymbol\beta}_{\alpha_\lambda})}
\frac{L(\hat{\boldsymbol\beta}_{\alpha_\lambda})}{L(\hat{\boldsymbol\beta}_\lambda)}
\frac{\tilde{R}(\hat{\boldsymbol\beta}_{\alpha_\lambda})}{L(\hat{\boldsymbol\beta}_{\alpha_\lambda})} \right ]^{1/2} \\
& + \left [ \frac{\sigma^2 d_{\alpha_\lambda}}{n\tilde{R}(\hat{\boldsymbol\beta}^*_{\alpha_\lambda})}
\frac{L(\hat{\boldsymbol\beta}_{\alpha_\lambda})}{L(\hat{\boldsymbol\beta}_\lambda))}
\frac{\tilde{R}(\hat{\boldsymbol\beta}_{\alpha_\lambda})}{L(\hat{\boldsymbol\beta}_{\alpha_\lambda})} \right ]
\left | \frac{||\boldsymbol\varepsilon||^2}{n\sigma^2} - 1 \right |
+\frac{d_{\alpha_\lambda}}{n}
\frac{||\hat{\boldsymbol\mu}_{\alpha_\lambda}-\hat{\boldsymbol\mu}_\lambda||^2}{nL(\hat{\boldsymbol\beta}_\lambda)}.
\end{align*}}
By definition,
{\small
\begin{align*}
\tilde{R}(\hat{\boldsymbol\beta}_{\alpha_\lambda})
 \geq \frac{\sigma^2 d_{\alpha_\lambda}}{n}.
\end{align*}}
Thus
{\small
\begin{align}\label{estVar-finalResult}
\frac{|\hat{\sigma}^2_\lambda - \sigma^2|d_{\alpha_\lambda}}{nL(\hat{\boldsymbol\beta}_\lambda)}
& \leq
\sup_{\lambda \in [0,\lambda_{max}]}  \frac{d_{\alpha_\lambda}}{n} \left [\frac{L(\hat{\boldsymbol\beta}_{\alpha_\lambda})}{L(\hat{\boldsymbol\beta}_\lambda)} \right]\\ \nonumber
& + \sup_{\lambda \in [0,\lambda_{max}]} \frac{2}{\sigma} \left [ \frac{d_{\alpha_\lambda}}{n} \frac{||\boldsymbol\varepsilon||^2}{n} \right ]^{1/2}
\left [ \frac{L(\hat{\boldsymbol\beta}_{\alpha_\lambda})}{L(\hat{\boldsymbol\beta}_\lambda)} \right ]^{1/2}
\left [ \frac{L(\hat{\boldsymbol\beta}_{\alpha_\lambda})}{L(\hat{\boldsymbol\beta}_\lambda)}
\frac{\tilde{R}(\hat{\boldsymbol\beta}_{\alpha_\lambda})}{L(\hat{\boldsymbol\beta}_{\alpha_\lambda})} \right ]^{1/2} \\ \nonumber
& +  \sup_{\lambda \in [0,\lambda_{max}]} \left [ \frac{L(\hat{\boldsymbol\beta}_{\alpha_\lambda})}{L(\hat{\boldsymbol\beta}_\lambda)}
\frac{\tilde{R}(\hat{\boldsymbol\beta}_{\alpha_\lambda})}{L(\hat{\boldsymbol\beta}_{\alpha_\lambda})} \right ]
\left | \frac{||\boldsymbol\varepsilon||^2}{n\sigma^2} - 1 \right |
+\sup_{\lambda \in [0,\lambda_{max}]}\frac{d_{\alpha_\lambda}}{n}
\frac{||\hat{\boldsymbol\mu}_{\alpha_\lambda}-\hat{\boldsymbol\mu}_\lambda||^2}{nL(\hat{\boldsymbol\beta}_\lambda)} . \end{align}}
\citet{li87} established that
{\small
\begin{align*}
\sup_{\alpha \in \mathcal{A}_n} \left | \frac{L(\hat{\boldsymbol\beta}_\alpha)}{R(\hat{\boldsymbol\beta}_\alpha)} -1 \right| \rightarrow_p 0
\end{align*}}
and it follows that
{\small
\begin{align}\label{liresult}
\sup_{\alpha \in \mathcal{A}_n} \left | \frac{L(\hat{\boldsymbol\beta}_{\alpha_\lambda})}{\tilde{R}(\hat{\boldsymbol\beta}_{\alpha_\lambda})} -1 \right| \rightarrow_p 0.
\end{align}}
In addition, from the proof of Theorem 2 in ZLT we have that
{\small
\begin{align}\label{ZLTresult1}
\sup_{\lambda \in [0,\lambda_{max}]} \left | \frac{L(\hat{\boldsymbol\beta}_{\alpha_\lambda})-L(\hat{\boldsymbol\beta}_\lambda)}{L(\hat{\boldsymbol\beta}_\lambda)} \right | \rightarrow_p 0,
\end{align}}
and
{\small
\begin{align}\label{ZLTresult2}
\sup_{\lambda \in [0,\lambda_{max}]} \frac{||\hat{\boldsymbol\mu}_{\alpha_\lambda}-\hat{\boldsymbol\mu}_\lambda||^2}{nL(\hat{\boldsymbol\beta}_\lambda)} \rightarrow_p 0.
\end{align}}
Combining these results with the Law of Large Numbers and the assumption that $d_n/n \rightarrow 0$ as $n \rightarrow \infty$ the four terms on the right-hand side of equation (\ref{estVar-finalResult}) converge to $0$ in probability.  Hence,
{\small
\begin{align*}
\sup_{\lambda \in [0,\lambda_{max}]}  \frac{2d_{\alpha_\lambda}|\hat{\sigma}^2_\lambda - \sigma^2|}{nL(\hat{\boldsymbol\beta}_\lambda)}  \rightarrow_p 0
\end{align*}
}
as desired.
\end{proof}

\begin{lemma}
Assume that (A1)-(A4) hold and that $d_n/n \rightarrow 0$ as $n \rightarrow \infty$. Then
\begin{align*}
\sup_{\lambda \in [0,\lambda_{max}]} \frac{ d_{\alpha_\lambda}|\tilde{\sigma}^2_n - \sigma^2| }{nL(\hat{\beta}_\lambda)}
    \rightarrow_p 0.
\end{align*}
\end{lemma}

\begin{proof}
Start by noting that for all $\lambda \in [0,\lambda_{max}]$, $\Delta_{\alpha_\lambda} \geq \Delta_{\bar{\alpha}}$ (ZLT).  Consider
{\small
\begin{align*}
 \frac{ \tilde{R}(\hat{\boldsymbol\beta}_{\bar{\alpha}})d_{\alpha_\lambda}}{\tilde{R}(\hat{\boldsymbol\beta}_{\alpha_\lambda})d_n}
& \leq \frac{ ( \Delta_{\bar{\alpha}} + \frac{d_n \sigma^2}{n} ) d_{\alpha_\lambda} }{ (\Delta_{\bar{\alpha}} + \frac{d_{\alpha_\lambda} \sigma^2}{n})  d_n} \\
& \leq \frac{ \Delta_{\bar{\alpha}}} { \Delta_{\bar{\alpha}} + \frac{d_{\alpha_\lambda} \sigma^2}{n} }
+ \frac{ \frac{d_n \sigma^2}{n} d_{\alpha_\lambda} }{ \frac{d_{\alpha_\lambda} \sigma^2}{n} d_n(\bar\alpha) } \\
&  \leq  2.
\end{align*}}
From the proof of Lemma 1 we have that
{\small
\begin{align*}
|\tilde{\sigma}^2_n - \sigma^2| & \leq \frac{n}{n-d_n-1} L(\hat{\boldsymbol\beta}^*_{\bar{\alpha}}) + 2\frac{n}{n-d_n-1}||\boldsymbol\varepsilon_n||\frac{||\boldsymbol\mu-\hat{\boldsymbol\mu}_{\bar{\alpha}}||}{n}
+\left | \frac{||\boldsymbol\varepsilon||^2}{n-d_n-1} - \sigma^2 \right |.
\end{align*}}
Thus
{\small
\begin{align*}
\frac{d_n|\tilde{\sigma}^2 - \sigma^2|}{nL(\hat{\boldsymbol\beta}_{\bar{\alpha}})}
& \leq \frac{n}{n-d_n-1}\frac{d_n}{n} + \frac{2}{\sigma}\frac{n}{n-d_n-1}\left[ \frac{||\boldsymbol\varepsilon||^2}{n}\frac{d_n}{n} \right ]^{1/2}\left [ \frac{\sigma^2 d_n}{n\tilde{R}(\hat{\boldsymbol\beta}_{\bar{\alpha}})} \frac{\tilde{R}(\hat{\boldsymbol\beta}_{\bar{\alpha}})}{L(\hat{\boldsymbol\beta}_{\bar{\alpha}})}\right]^{1/2}\\
& \left[\frac{d_n\sigma^2}{n\tilde{R}(\hat{\boldsymbol\beta}_{\bar{\alpha}})}\frac{\tilde{R}(\hat{\boldsymbol\beta}_{\bar{\alpha}})}
{L(\hat{\boldsymbol\beta}_{\bar{\alpha}})}\right ]\left | \frac{||\boldsymbol\varepsilon||^2}{(n-d_n-1)\sigma^2} - 1 \right |.
\end{align*}}
Under the assumption that $d_n/n \rightarrow 0$ as $n \rightarrow \infty$ it follows that
{\small
\begin{align*}
\frac{d_n|\tilde{\sigma}^2 - \sigma^2|}{nL(\hat{\boldsymbol\beta}_{\bar{\alpha}})} \rightarrow_p 0.
\end{align*}}
Combining these results with (\ref{liresult}) and (\ref{ZLTresult1}) it follows that
{\small
\begin{align*}
\frac{ d_{\alpha_\lambda}|\tilde{\sigma}^2 - \sigma^2| }{n L(\hat{\beta}_\lambda)}
& \leq \sup_{[0,\lambda_{max}]} \frac{ d_n|\tilde{\sigma}^2 - \sigma^2| }{n L(\hat{\boldsymbol\beta}_{\bar{\alpha}})}
\frac{d_{\alpha_\lambda}\tilde{R}(\hat{\boldsymbol\beta}_{\bar{\alpha}})}{d_n
\tilde{R}(\hat{\boldsymbol\beta}_{\alpha_\lambda})}
\frac{L(\hat{\boldsymbol\beta}_{\bar{\alpha}})}{\tilde{R}(\hat{\boldsymbol\beta}_{\bar{\alpha}})}
\frac{\tilde{R}(\hat{\boldsymbol\beta}_{\alpha_\lambda})}
{L(\hat{\boldsymbol\beta}_{\alpha_\lambda})}\frac{L(\hat{\boldsymbol\beta}_{\alpha_\lambda})}{{L(\hat{\boldsymbol\beta}_\lambda)}}\\
& \leq 2  \frac{ d_n|\tilde{\sigma}^2  - \sigma^2| }{n L(\hat{\boldsymbol\beta}_{\bar{\alpha}})}
\sup_{[0,\lambda_{max}]}
\frac{L(\hat{\boldsymbol\beta}_{\bar{\alpha}})}{\tilde{R}(\hat{\boldsymbol\beta}_{\bar{\alpha}})}
\frac{\tilde{R}(\hat{\boldsymbol\beta}_{\alpha_\lambda})}
{L(\hat{\boldsymbol\beta}_{\alpha_\lambda})}\frac{L(\hat{\boldsymbol\beta}_{\alpha_\lambda})}{{L(\hat{\boldsymbol\beta}_\lambda)}}
\rightarrow_p  0.
\end{align*}}
\end{proof}
\begin{proof}[Proof of Theorem 1]
As in the proofs in ZLT, to prove that $C_{p_\lambda}$ is asymptotically loss efficient, it is sufficient to show that
{\small
\begin{align}\label{thm1result}
 \sup_{\lambda \in [0,\lambda_{max}]} \left | \frac{C_{p_\lambda} - ||\boldsymbol\varepsilon||^2/n - L(\hat{\boldsymbol\beta}_\lambda)}{L(\hat{\boldsymbol\beta}_\lambda)} \right | \rightarrow_p 0.
\end{align}}
Decomposing $C_{p_\lambda}$ it can be established that
{\small
\begin{align*}
C_{p_\lambda}& = \frac{||\mathbf{y}-\hat{\boldsymbol\mu}_\lambda||^2}{n} + \frac{2\tilde{\sigma}^2 d_{\alpha_\lambda}}{n} \\
  & = \frac{||\boldsymbol\varepsilon||^2}{n} + L(\hat{\boldsymbol\beta}_\lambda)+(L(\hat{\boldsymbol\beta}_{\alpha_\lambda})-L(\hat{\boldsymbol\beta}_\lambda))+
\frac{||\hat{\boldsymbol\mu}_{\alpha_\lambda}-\hat{\boldsymbol\mu}_\lambda||^2}{n} \\
& + \frac{2\boldsymbol{\varepsilon}^T[I-\mathbf{H}_{\alpha_\lambda}]\boldsymbol\mu}{n} +
\frac{2(\sigma^2 d_{\alpha_\lambda} - \boldsymbol{\varepsilon}^T \mathbf{H}_{\alpha_\lambda} \boldsymbol{\varepsilon})}{n}
+ \frac{2(\tilde{\sigma}^2  - \sigma^2)d_{\alpha_\lambda}}{n}.
\end{align*}}
The proof of Theorem 2 in ZLT established that
{\small
 \begin{align*}
 \sup_{\lambda \in [0,\lambda_{max}]} \left | \frac{2\boldsymbol{\varepsilon}^T(I-\mathbf{H}_{\alpha_\lambda})\boldsymbol{\mu}}{nL(\hat{\boldsymbol\beta}_\lambda)} \right | \rightarrow_p 0,
 \end{align*}}
and,
{\small
\begin{align*}
 \sup_{\lambda \in [0,\lambda_{max}]} \left | \frac{2(\sigma^2 d_{\alpha_\lambda} - \boldsymbol{\varepsilon}_n^T \mathbf{H}_{\alpha_\lambda} \boldsymbol{\varepsilon})}{nL(\hat{\boldsymbol\beta}_\lambda)} \right | \rightarrow_p 0.
 \end{align*}}
Combining these results with (\ref{liresult})-(\ref{ZLTresult2}) and Lemma 2, (\ref{thm1result}) follows as desired.
\end{proof}

\begin{proof}[Proof of Theorem 2]
The proof is the same as that of Theorem 1 except that the estimated variance is based on the candidate model rather than the full model and the result is established by using Lemma 1 in place of Lemma 2.
\end{proof}

\begin{proof}[Proof of Theorem 3]
As in the efficiency proof for $\Gamma_\lambda$, it is sufficient to show that
{\small
\begin{align}\label{thm3result}
 \sup_{\lambda \in [0,\lambda_{max}]} \left | \frac{\tilde{\Gamma}_\lambda - ||\boldsymbol\varepsilon||^2/n - L(\hat{\boldsymbol\beta}_\lambda)}{L(\hat{\boldsymbol\beta}_\lambda)} \right | \rightarrow_p 0
\end{align}}
\noindent to establish that $\tilde{\Gamma}_\lambda$ is an asymptotically efficient selection procedure for the regularization parameter, $\lambda$.
By the definition of $\tilde{\Gamma}_\lambda$ we have that
{\small
\begin{align*}
\sup_{\lambda \in [0,\lambda_{max}]} \left | \frac{\tilde{\Gamma}_\lambda - ||\boldsymbol\varepsilon||^2/n - L(\hat{\boldsymbol\beta}_\lambda)}{L(\hat{\boldsymbol\beta}_\lambda)} \right |
& = \sup_{\lambda \in [0,\lambda_{max}]} \left | \frac{\delta_\lambda\hat{\sigma}^2_\lambda+\Gamma_\lambda - ||\boldsymbol\varepsilon||^2/n - L(\hat{\boldsymbol\beta}_\lambda)}{L(\hat{\boldsymbol\beta}_\lambda)} \right | \nonumber \\
&\leq \sup_{\lambda \in [0,\lambda_{max}]} \left | \frac{\delta_\lambda ( \hat{\sigma}^2_\lambda-\sigma^2)}{L(\hat{\boldsymbol\beta}_\lambda)} \right |
+  \sup_{\lambda \in [0,\lambda_{max}]} \frac{|\delta_\lambda| \sigma^2}{L(\hat{\boldsymbol\beta}_\lambda)} \nonumber \\
&+ \sup_{\lambda \in [0,\lambda_{max}]} \left |\frac{\Gamma_\lambda - ||\boldsymbol\varepsilon||^2/n - L(\hat{\boldsymbol\beta}_\lambda)}{L(\hat{\boldsymbol\beta}_\lambda)} \right |.
\end{align*}}
The last two terms converge to zero by (\ref{assump1}) and the efficiency proof for $\Gamma_\lambda$.  From the proof of the previous lemma we further have that
{\small
\begin{align*}
\left | \frac{\delta_\lambda(\hat{\sigma}^2_\lambda-\sigma^2)}{L(\hat{\boldsymbol\beta}_\lambda)} \right |
& \leq
|\delta_\lambda| \frac{ L(\hat{\boldsymbol\beta}_{\alpha_\lambda})}{L(\hat{\boldsymbol\beta}_\lambda)}
+ 2 \frac{||\boldsymbol\varepsilon||}{\sqrt{n}} \left(\frac{L(\hat{\boldsymbol\beta}_{\alpha_\lambda})}{L(\hat{\boldsymbol\beta}_\lambda)} \right)^{1/2}
\left ( \frac{|\delta_\lambda|}{L(\hat{\boldsymbol\beta}_\lambda)} \right)^{1/2} (|\delta_\lambda|)^{1/2} \nonumber \\
&+ \frac{|\delta_\lambda|}{L(\hat{\boldsymbol\beta}_\lambda)}\left | \frac{||\boldsymbol\varepsilon||^2}{n} - \sigma^2 \right | +  |\delta_\lambda|
\frac{||\hat{\boldsymbol\mu}_{\alpha_\lambda}-\hat{\boldsymbol\mu}_{\lambda}||^2}{nL(\hat{\boldsymbol\beta}_\lambda)}. \nonumber \end{align*}}
By (\ref{assump1}), (\ref{assump2}), and similar arguments as those used in the efficiency proof for $\Gamma_\lambda$ we have that the right hand side converges to $0$ in probability.  Therefore, it follows that
{\small
\begin{align*}
\sup_{\lambda \in [0,\lambda_{max}]} \left | \frac{\delta_\lambda(\hat{\sigma}^2_\lambda-\sigma^2)}{L(\hat{\boldsymbol\beta}_\lambda)} \right | \rightarrow_p 0
\end{align*}}
and so equation (\ref{thm3result}) holds as desired.
\end{proof}

\appsection{}
\begin{proof}[Proof of Lemma \ref{l:A4glm}]
Under assumptions (A5$'$)-(A7$'$),
\[
\sup_{\lambda \in [0,\lambda_{\max}]} \frac{ ||\vb||^2 }{\tilde{R}_{KL}(\hat\vbeta_{\alpha_\lambda})}
\leq \frac{M_2\lambda^2_{\max} d}{L_{KL}(\vbeta^*_{\bar{\alpha}})}
\leq \frac{M_1^2 M_2 d}{nL_{KL}(\vbeta^*_{\bar{\alpha}})} \rightarrow 0.
\]
\end{proof}

\begin{lemma}\label{Lkl}
Under (R1)-(R5), for $n$ sufficiently large
\[
L_{KL}(\hat{\boldsymbol\beta}_\alpha)
    = L_{KL}(\boldsymbol\beta^*_\alpha) + \frac{1}{n}||\mathbf{W}^{1/2}_\alpha\mathbf{H}_\alpha(\mathbf{y}-\boldsymbol\mu)||^2
    + O_p(||\hat{\boldsymbol\beta}_\alpha-\boldsymbol\beta^*_\alpha||^2).
\]
\end{lemma}

\begin{proof}
Taylor's expansion of $b(\hat{\boldsymbol\theta}_\alpha)$ around $\boldsymbol\theta^*_\alpha$ gives us
{\small
\begin{align*}
\vone^T b(\hat{\boldsymbol\theta}_\alpha)
   & = \vone^T  b(\boldsymbol\theta^*_\alpha)
    + b'(\boldsymbol\theta^*_\alpha)^T(\hat{\boldsymbol\theta}_\alpha-\boldsymbol\theta^*_\alpha)\\
   & + \frac{1}{2}(\hat{\boldsymbol\theta}_\alpha-\boldsymbol\theta^*_\alpha)^T
        \mathbf{W}_\alpha
            (\hat{\boldsymbol\theta}_\alpha-\boldsymbol\theta^*_\alpha)
    + o_p(||\hat{\boldsymbol\theta}_\alpha-\boldsymbol\theta^*_\alpha||^2).
\end{align*}}
For $n$ sufficiently large, we have that
{\small
\begin{align*}
L_{KL}(\hat{\boldsymbol\beta}_\alpha)
    & = \frac{2}{n} \boldsymbol\mu^T(\vtheta_0 - \hat\vtheta_\alpha + \frac{2}{n} \vone^T  b(\hat{\boldsymbol\theta}_\alpha)\\
    & = L_{KL}(\boldsymbol\beta^*_\alpha)-\frac{2}{n} (\boldsymbol\mu-b'(\boldsymbol\theta^*_\alpha))^T (\hat{\boldsymbol\theta}_\alpha-\boldsymbol\theta^*_\alpha)
        + \frac{1}{n} (\hat{\boldsymbol\theta}_\alpha-\boldsymbol\theta^*_\alpha)^T \mathbf{W}_\alpha (\hat{\boldsymbol\theta}_\alpha-\boldsymbol\theta^*_\alpha)\\
    &    + o_p(||\hat{\boldsymbol\beta}_\alpha-\boldsymbol\beta^*_\alpha||^2)\\
    & = L_{KL}(\boldsymbol\beta^*_\alpha) +  \frac{1}{n}||\mathbf{W}^{1/2}_\alpha\mathbf{H}_\alpha(\mathbf{y}-\boldsymbol\mu)||^2
    + O_p(||\hat{\boldsymbol\beta}_\alpha-\boldsymbol\beta^*_\alpha||^2),
\end{align*}}
where the last equality follows from equations (\ref{pseudoProp}) and (\ref{estMLE}).
\end{proof}

\begin{lemma}\label{thm4Support}
Under assumptions (A1$'$)-(A4$'$), (A7$'$) and regularity conditions (R1)-(R3), the following results hold.
\begin{equation}\label{r1}
\sup_{\alpha \in \mathcal{A}_n}
\left| \frac{(\vy - \vmu)^T (\vtheta^*_\alpha - \vtheta_0)}
    {nR_{KL}(\hat{\boldsymbol\beta}_\alpha)} \right| \rightarrow_p 0,
\end{equation}
\begin{equation}\label{r4}
\sup_{\alpha \in \mathcal{A}_n}
    \left| \frac{(d_\alpha - tr\{(\mathbf{X}'_\alpha \mathbf{W}_\alpha \mathbf{X}_\alpha)^{-1} \mathbf{X}'_\alpha \mathbf{W}_0 \mathbf{X}_\alpha \})}
        {nR_{KL}(\hat{\boldsymbol\beta}_\alpha)} \right| \rightarrow_p 0,
\end{equation}
\begin{equation}\label{r3}
\sup_{\alpha \in \mathcal{A}_n}
    \left| \frac{((\mathbf{y}-\boldsymbol\mu)'\mathbf{H}_\alpha (\mathbf{y} - \boldsymbol\mu)
        - tr\{(\mathbf{X}'_\alpha \mathbf{W}_\alpha \mathbf{X}_\alpha)^{-1} \mathbf{X}'_\alpha \mathbf{W}_0 \mathbf{X}_\alpha \})}
        {nR_{KL}(\hat{\boldsymbol\beta}_\alpha)} \right| \rightarrow_p 0,
\end{equation}
and
\begin{equation}\label{r5}
\sup_{\alpha \in \mathcal{A}_n}
    \left| \frac{L_{KL}(\hat{\boldsymbol\beta}_\alpha)}
        {R_{KL}(\hat{\boldsymbol\beta}_\alpha)} - 1 \right| \rightarrow_p 0.
\end{equation}
\end{lemma}

The proof of this lemma requires the following matrix algebra results.
\begin{mydef}
Let $\mathbf{A}$ and $\mathbf{B}$ be two $K \times K$ matrices. We say that $\mathbf{A} \geq \mathbf{B}$ if $\mathbf{A}-\mathbf{B}$ is positive semidefinite.
\end{mydef}

\begin{lemma} \citep[p.471]{horn85}
If $\mathbf{A}$ and $\mathbf{B}$ are $K \times K$ positive definite Hermitian matrices, then
\begin{enumerate}
\item[(i.)] $\mathbf{A} \geq \mathbf{B}$ if and only if $\mathbf{B}^{-1} \geq A^{-1}$;
\item[(ii.)] if $\mathbf{A} \geq \mathbf{B}$, then $\lambda_k(\mathbf{A}) \geq \lambda_k(\mathbf{B})$ for all $k=1, \ldots, K$, where $\lambda_k(\mathbf{A})$ and $\lambda_k(\mathbf{B})$ are the $k^{th}$ largest eigenvalues of $\mathbf{A}$ and $\mathbf{B}$, respectively.
\end{enumerate}
\end{lemma}

\begin{lemma} \citep[p.340]{marshall10}
If $\mathbf{A}$ and $\mathbf{B}$ are $K \times K$ positive semidefinite Hermitian matrices, then
\[
tr(\mathbf{A}\mathbf{B}) \leq \sum_{k=1}^K \lambda_k(\mathbf{A}) \lambda_k(\mathbf{B}).
\]
\end{lemma}

\begin{proof}[Proof of Lemma \ref{thm4Support}]
We start by proving equation (\ref{r1}).  By Chebyshev's Inequality and Theorem 2 of \citet{whittle60}, we have that
\begin{equation}\label{r1.f1}
\Pr \left(
\sup_{\alpha \in \mathcal{A}_n}
\left|
\frac{(\vy - \vmu)^T (\vtheta^*_\alpha - \vtheta_0)}
{nR_{KL}(\hat{\boldsymbol\beta}_\alpha)}
\right|
>\delta
\right)
    \leq
        \frac{C}{\delta^{2q}}
        \sum_{\alpha \in \mathcal{A}_n}
        \frac{||\vtheta^*_\alpha - \vtheta_0||^{2q}}
        { (nR_{KL}(\hat\vbeta_\alpha))^{2q} }.
\end{equation}
Now $R_{KL}(\hat\vbeta_\alpha) \geq L_{KL}(\vbeta^*_\alpha)$.  If we consider $L_{KL}(\cdot)$ as a function of $\vtheta$, then by a second order Taylor series expansion around $\vtheta_0$,
\[
L_{KL}(\vtheta^*_\alpha) = \frac{1}{n}(\vtheta^*_\alpha -\vtheta_0)^T \bar{\mW} (\vtheta^*_\alpha -\vtheta_0),
\]
where $\bar{\mW} = diag\{b''(\bar\theta_1),\ldots,b''(\bar\theta_n)\}$ and $\bar{\theta}_i$ is on the line segment between $\theta^*_{\alpha i}$ and $\theta_{0 i}$.  Since $b''(\theta)>0$ for all $\theta$, it follows that $nR_{KL}(\hat\vbeta_\alpha) \geq K||\vtheta^*_\alpha -\vtheta_0||^2$ for some constant $K>0$.  Therefore the right-hand side of equation (\ref{r1.f1}) is less than or equal to
\[
\frac{C'}{\delta^{2q}}
\sum_{\alpha \in \mathcal{A}_n}
(n R_{KL}(\hat\vbeta_\alpha))^{-q}
\]
for some constant $C'>0$, which tends to zero as $n \rightarrow \infty$ by assumption (A3$'$).
Next, to establish equation (\ref{r4}) we first note that
$\mathbf{X}'_\alpha \mathbf{W}_0 \mathbf{X}_\alpha \leq \max_{1 \leq i \leq n} \sigma^2_i \mathbf{X}'_\alpha \mathbf{X}_\alpha$
and
$\mathbf{X}'_\alpha \mathbf{W}_\alpha \mathbf{X}_\alpha \geq \min_{1 \leq i \leq n} b''(\theta_{\alpha i}) \mathbf{X}'_\alpha \mathbf{X}_\alpha$.  From Lemmas B.2 and B.3 it follows then that
{\small
\[
tr((\mathbf{X}'_\alpha \mathbf{W}_\alpha \mathbf{X}_\alpha)^{-1} \mathbf{X}'_\alpha \mathbf{W}_0 \mathbf{X}_\alpha)
    \leq
         d_\alpha
        \frac{\max_{1 \leq i \leq n} \sigma^2_i}
            {\min_{1 \leq i \leq n} b''(\theta_{\alpha i})}
       \lambda_1\left(\left(\frac{1}{n}\mathbf{X}'_\alpha \mathbf{X}_\alpha\right)^{-1}\right)
            \lambda_1 \left( \frac{1}{n}\mathbf{X}'_\alpha \mathbf{X}_\alpha \right)
    \leq d_\alpha C
\]}
for some constant $C>0$.  Using this result we have that
{\small
\[
\sup_{\alpha \in \mathcal{A}_n}
    \left| \frac{2(d_\alpha
        - tr\{(\mathbf{X}'_\alpha \mathbf{W}_\alpha \mathbf{X}_\alpha)^{-1} \mathbf{X}'_\alpha \mathbf{W}_0 \mathbf{X}_\alpha \})}
        {nR_{KL}(\hat{\boldsymbol\beta}_\alpha)} \right|
            \leq
                \sup_{\alpha \in \mathcal{A}_n}
                    \frac{2 d_\alpha (1+C)}{nR_{KL}(\hat{\boldsymbol\beta}_\alpha)}
            \leq \frac{d_n(1+C)}{nL_{KL}(\vbeta^*_{\bar{\alpha}})},
\]}
which tends to zero by assumption (A7$'$).

To prove equation (\ref{r3}) we apply Chebyshev's Inequality and Theorem 2 of \citet{whittle60} to get that
{\small
\begin{multline*}
\Pr \left( \sup_{\alpha \in \mathcal{A}_n}
    \left| \frac{2((\mathbf{y}-\boldsymbol\mu)'\mathbf{H}_\alpha(\mathbf{y} - \boldsymbol\mu)
        - tr\{(\mathbf{X}'_\alpha \mathbf{W}_\alpha \mathbf{X}_\alpha)^{-1} \mathbf{X}'_\alpha \mathbf{W}_0 \mathbf{X}_\alpha \})}
        {nR_{KL}(\hat{\boldsymbol\beta}_\alpha)} \right| > \delta \right)\\
            \leq \delta^{-2q} C \sum_{\alpha \in \mathcal{A}_n}
                    \frac{ tr\{
                        (\mathbf{X}'_\alpha \mathbf{W}_\alpha \mathbf{X}_\alpha)^{-1} \mathbf{X}'_\alpha \mathbf{W}_0 \mathbf{X}_\alpha (\mathbf{X}'_\alpha \mathbf{W}_\alpha \mathbf{X}_\alpha)^{-1} \mathbf{X}'_\alpha \mathbf{W}_0 \mathbf{X}_\alpha \}^q}
                    {(nR_{KL}(\hat{\boldsymbol\beta}_\alpha))^{2q}}
\end{multline*}}
for some constant $C>0$.
Using the fact that $tr\{\mathbf{AB}\} \leq \lambda_1(\mathbf{A}) tr\{\mathbf{B}\}$,
{\small
\[
tr\{(\mathbf{X}'_\alpha \mathbf{W}_\alpha \mathbf{X}_\alpha)^{-1} \mathbf{X}'_\alpha \mathbf{W}_0 \mathbf{X}_\alpha (\mathbf{X}'_\alpha \mathbf{W}_\alpha \mathbf{X}_\alpha)^{-1} \mathbf{X}'_\alpha \mathbf{W}_0 \mathbf{X}_\alpha \}
    \leq K tr\{(\mathbf{X}'_\alpha \mathbf{W}_\alpha \mathbf{X}_\alpha)^{-1} \mathbf{X}'_\alpha \mathbf{W}_0 \mathbf{X}_\alpha\}
\]}
for some constant $K>0$.
Therefore
{\small
\begin{multline*}
\Pr \left( \sup_{\alpha \in \mathcal{A}_n}
    \left| \frac{2((\mathbf{y}-\boldsymbol\mu)'\mathbf{H}_\alpha(\mathbf{y} - \boldsymbol\mu)
        - tr\{(\mathbf{X}'_\alpha \mathbf{W}_\alpha \mathbf{X}_\alpha)^{-1} \mathbf{X}'_\alpha \mathbf{W}_0 \mathbf{X}_\alpha \})}
        {nR_{KL}(\hat{\boldsymbol\beta}_\alpha)} \right| > \delta \right)\\
            \leq \delta^{-2q} C'
                \sum_{\alpha \in \mathcal{A}}
                    \frac{ tr\{(\mathbf{X}'_\alpha \mathbf{W}_\alpha \mathbf{X}_\alpha)^{-1} \mathbf{X}'_\alpha \mathbf{W}_0 \mathbf{X}_\alpha\}^q }
                    {(nR_{KL}(\hat{\boldsymbol\beta}_\alpha))^{2q}}.
\end{multline*}}
for some constant $C'>0$.
Since
{\small
\[
\frac{ tr\{(\mathbf{X}'_\alpha \mathbf{W}_\alpha \mathbf{X}_\alpha)^{-1} \mathbf{X}'_\alpha \mathbf{W}_0 \mathbf{X}_\alpha\} }{n}
    \leq R_{KL}(\hat{\boldsymbol\beta}_\alpha),
\]}
it follows that
{\small
\begin{multline*}
\Pr \left( \sup_{\alpha \in \mathcal{A}_n}
    \left| \frac{2((\mathbf{y}-\boldsymbol\mu)'\mathbf{H}_\alpha(\mathbf{y} - \boldsymbol\mu))
        - tr\{(\mathbf{X}'_\alpha \mathbf{W}_\alpha \mathbf{X}_\alpha)^{-1} \mathbf{X}'_\alpha \mathbf{W}_0 \mathbf{X}_\alpha \})}
        {nR_{KL}(\hat{\boldsymbol\beta}_\alpha)} \right| > \delta \right)\\
            \leq \delta^{-2q} C'
                \sum_{\alpha \in \mathcal{A}_n}
                    (nR_{KL}(\hat{\boldsymbol\beta}_\alpha))^{-q} \rightarrow 0.
\end{multline*}}
Finally, equation (\ref{r5}) follows from (\ref{r3}).
\end{proof}

\begin{lemma}\label{estBound}
Under (A1$'$)
\[
||\hat{\boldsymbol\theta}_\lambda - \hat{\boldsymbol\theta}_{\alpha_\lambda}||^2 \leq n C ||\mathbf{b}||^2.
\]
\end{lemma}

\begin{proof}  $\hat{\boldsymbol\beta}_{\lambda }$ satisfies
{\small
\[
0 = \frac{1}{n}\frac{\partial l(\hat{\boldsymbol\beta}_{\lambda})}{\partial \boldsymbol\beta} - \mathbf{b}.
\]}
Without loss of generality, we can write $\hat{\boldsymbol\beta}_\lambda = (\hat{\boldsymbol\beta}_{\lambda 1}, \hat{\boldsymbol\beta}_{\lambda 2})'$ where $\hat{\boldsymbol\beta}_{\lambda 2} = \mathbf{0}$ and
$\hat{\boldsymbol\beta}_{\lambda 1}$ is a $1 \times d_{\alpha_\lambda}$ vector of estimated coefficients.  Applying the mean value theorem, we get that
{\small
\[
0 = \frac{1}{n}\frac{\partial l(\hat{\boldsymbol\beta}_{\alpha_\lambda})}{\partial \boldsymbol\beta}
    + \frac{1}{n}\frac{\partial^2 l(\bar{\boldsymbol\beta})}{\partial \boldsymbol\beta \partial \boldsymbol\beta^T}
        (\hat{\boldsymbol\beta}_{\lambda 1}- \hat{\boldsymbol\beta}_{\alpha_\lambda})
    - \mathbf{b}_1,
\]}
where $\bar{\boldsymbol\beta}$ is on the line segment joining $\hat{\boldsymbol\beta}_{\lambda 1}$ and $\hat{\boldsymbol\beta}_{\alpha_\lambda}$, and $\mathbf{b}_1$ are the non-zero components of $\mathbf{b}$ that correspond to $\hat{\boldsymbol\beta}_{\lambda 1}$.  For $n$ sufficiently large, it follows then that
{\small
\begin{equation}\label{MLE-Reg-Diff}
\hat{\boldsymbol\beta}_{\lambda 1} - \hat{\boldsymbol\beta}_{\alpha_\lambda}
    = \left(\frac{1}{n} \mathbf{X}'_{\alpha_\lambda} \bar{\mathbf{W}}_\alpha \mathbf{X}_{\alpha_\lambda} \right)^{-1} \mathbf{b}_1,
\end{equation}}
where $\bar{\mathbf{W}}_\alpha = diag\{b''(\bar{\theta}_1),\ldots,b''(\bar{\theta}_n)\}$.  Therefore
{\small
\[
||\hat{\boldsymbol\theta}_\lambda - \hat{\boldsymbol\theta}_{\alpha_\lambda}||^2
    = ||X_{\alpha_\lambda}(\hat{\boldsymbol\beta}_{\lambda 1} - \hat{\boldsymbol\beta}_{\alpha_\lambda})||^2
    = n \mathbf{b}'_1 \left(\frac{1}{n} \mathbf{X}'_{\alpha_\lambda} \bar{\mathbf{W}}_\alpha \mathbf{X}_{\alpha_\lambda} \right)^{-1}
        \left(\frac{1}{n} \mathbf{X}'_{\alpha_\lambda} \mathbf{X}_{\alpha_\lambda} \right)
        \left(\frac{1}{n} \mathbf{X}'_{\alpha_\lambda} \bar{\mathbf{W}}_\alpha \mathbf{X}_{\alpha_\lambda} \right)^{-1}
        \mathbf{b}_1.
\]}
Since
{\small
\begin{equation}\label{useful-bound}
\left(\frac{1}{n} \mathbf{X}'_{\alpha_\lambda} \bar{\mathbf{W}}_\alpha \mathbf{X}_{\alpha_\lambda} \right)^{-1}
        \left(\frac{1}{n} \mathbf{X}'_{\alpha_\lambda} \mathbf{X}_{\alpha_\lambda} \right)
        \left(\frac{1}{n} \mathbf{X}'_{\alpha_\lambda} \bar{\mathbf{W}}_\alpha \mathbf{X}_{\alpha_\lambda} \right)^{-1}
    \leq
    (\min_{1\leq i \leq n} b''(\bar{\theta}_{i}))^{-2} \left(\frac{1}{n} \mathbf{X}'_{\alpha_\lambda} \mathbf{X}_{\alpha_\lambda} \right)^{-1},
\end{equation}}
{\small
\[
||\hat{\boldsymbol\theta}_\lambda - \hat{\boldsymbol\theta}_{\alpha_\lambda}||^2 \leq n C ||\mathbf{b}||^2
\]}
by Lemma B.3 and assumption (A1$'$).
\end{proof}
Since $\mathcal{A}_n$ includes all subsets, the results in Lemma B.2 will still hold when the candidate model $\alpha$ is replaced by the random candidate model $\alpha_\lambda$.
\begin{lemma}\label{lemmaEq}
Under (A1$'$)-(A7$'$),
\begin{equation}\label{lossEq}
\sup_{\lambda \in [0,\lambda_{\max}]}
    \left| \frac{L_{KL}(\hat{\boldsymbol\beta}_{\lambda})}{L_{KL}(\hat{\boldsymbol\beta}_{\alpha_\lambda})}
        -1 \right|
        \rightarrow_p 0.
\end{equation}
\end{lemma}
\begin{proof}
Applying a second-order Taylor expansion, we get
{\small
\begin{align*}
L_{KL}(\hat{\boldsymbol\beta}_{\lambda})-L_{KL}(\hat{\boldsymbol\beta}_{\alpha_\lambda})
    & = - \frac{2}{n} \boldsymbol\mu'
        (\hat{\boldsymbol\theta}_{\lambda}-\hat{\boldsymbol\theta}_{\alpha_\lambda})
        + \frac{2}{n}(b(\hat{\boldsymbol\theta}_{\lambda})-b(\hat{\boldsymbol\theta}_{\alpha_\lambda})) \\
    & =  - \frac{2}{n} (\boldsymbol\mu - b'(\hat{\boldsymbol\theta}_{\alpha_\lambda}))'
        (\hat{\boldsymbol\theta}_{\lambda}-\hat{\boldsymbol\theta}_{\alpha_\lambda})
        + \frac{1}{n} (\hat{\boldsymbol\theta}_{\lambda}-\hat{\boldsymbol\theta}_{\alpha_\lambda})'
            \bar{\mathbf{W}}_\alpha  (\hat{\boldsymbol\theta}_{\lambda}-\hat{\boldsymbol\theta}_{\alpha_\lambda}) \\
    & =  \frac{2}{n} (\mathbf{y} - \boldsymbol\mu)'
        (\hat{\boldsymbol\theta}_{\lambda}-\hat{\boldsymbol\theta}_{\alpha_\lambda})
        + \frac{1}{n} (\hat{\boldsymbol\theta}_{\lambda}-\hat{\boldsymbol\theta}_{\alpha_\lambda})'
            \bar{\mathbf{W}}_\alpha  (\hat{\boldsymbol\theta}_{\lambda}-\hat{\boldsymbol\theta}_{\alpha_\lambda}),
\end{align*}}
where the last equality follows from the fact that $\hat{\boldsymbol\theta}_{\alpha_\lambda}$ is the maximum-likelihood estimator so
$ \mathbf{X}'_{\alpha_\lambda}(\mathbf{y} - b'(\hat{\boldsymbol\theta}_{\alpha_\lambda}))=0$.

By equation (\ref{MLE-Reg-Diff}) and assumptions (A5$'$) and (A6$'$), the first term is bounded by
\[
M_1\frac{2}{n}(\vy -\vmu)^T \frac{\mX_{\alpha_\lambda}}{\sqrt{n}}
(\frac{1}{n}\mX^T_{\alpha_\lambda}\bar{\mW}_{\alpha}\mX^T_{\alpha_\lambda})^{-1}
\vone
\]
where $\vone$ is a $d_{\alpha_\lambda} \times 1$ vector of ones.  Applying Chebyshev's Inequality and Theorem 2 of \citet{whittle60}, we have that
\[
\Pr
\left(
\sup_{\lambda \in [0,\lambda_{\max}]}
\frac{(\vy - \vmu)^T(\hat\vtheta_\lambda - \hat\vtheta_{\alpha_\lambda})}
{n R_{KL}(\vbeta_{\alpha_\lambda})}
> \delta \right)
    \leq
    \frac{C }{\delta^{2q}}
    \sum_{\alpha \in \mathcal{A}_n}
    \frac{ ||n^{-1/2}\mX_\alpha(\frac{1}{n}\mX^T_{\alpha}\bar{\mW}_{\alpha}\mX^T_{\alpha})^{-1}\vone||^{2q}}
    {n^{2q} R_{KL}(\hat\vbeta_\alpha)^{2q}}
\]
for some constant $C>0$.  By equation (\ref{useful-bound}) and assumption (A1$'$), this does not exceed
\[
\frac{C'}{\delta^{2q}}
\sum_{\alpha \in \mathcal{A}_n}
\frac{d^{q}_\alpha}
{n^{2q} R_{KL}(\hat\beta_{\alpha})^{2q}}.
\]
By (A6$'$), $d/nR_{KL}(\beta^*_{\bar\alpha}) \rightarrow 0$, so, for $n$ sufficiently large, $d_\alpha < nR_{KL}(\hat\beta_\alpha)$.  Therefore, the last quantity is less than or equal to
\[
\frac{C'}{\delta^{2q}}
\sum_{\alpha \in \mathcal{A}_n}
(n R_{KL}(\hat\beta_{\alpha}))^{-q},
\]
which tends to zero by assumption (A3$'$).  Thus
\[
\sup_{\lambda \in [0, \lambda_{\max}] }
\frac{2(\mathbf{y} - \boldsymbol\mu)^T(\hat{\boldsymbol\theta}_{\lambda}-\hat{\boldsymbol\theta}_{\alpha_\lambda})}
{n L_{KL}(\hat\vbeta_{\alpha_\lambda})}
\rightarrow_p 0.
\]
Assuming that (A4$'$)-(A7$'$) holds, equation (\ref{lossEq}) follows from this result and Lemma \ref{estBound}.
\end{proof}
\begin{proof}
To prove the efficiency of $AIC_\lambda$, it suffices to show that
\begin{equation}\label{thm4result}
\sup_{\lambda \in [0,\lambda_{\max}]}
\left | \frac{ AIC_\lambda - \frac{2}{n}\vy^T\vtheta_0 + \frac{2}{n}\vone^Tb(\theta_0) -L_{KL}(\hat{\boldsymbol\beta}_\lambda) } {L_{KL}(\hat{\boldsymbol\beta}_\lambda) } \right|
\rightarrow_p 0.
\end{equation}
Consider
{\small
\begin{align*}
AIC_\lambda - \frac{2}{n}\vy^T\vtheta_0 + \frac{2}{n}\vone^Tb(\vtheta_0)
& = \frac{2}{n} \vy^T(\vtheta_0 - \hat\vtheta_\lambda) + \frac{2}{n} \vone^T(b(\hat\vtheta_\lambda) - b(\vtheta_0))
+ 2\frac{d_{\alpha_\lambda}}{n}\\
    & = L_{KL}(\hat\vbeta_\lambda) + \frac{2}{n}(\vy-\vmu)^T(\vtheta_0 - \vtheta^*_{\alpha_\lambda})\\
    &+ \frac{2}{n}(\vy-\vmu)^T(\vtheta^*_{\alpha_\lambda}-\hat\vtheta_{\alpha_\lambda})
    + \frac{2}{n}(\vy-\vmu)^T(\hat\vtheta_{\alpha_\lambda}-\hat\vtheta_\lambda) + 2\frac{d_{\alpha_\lambda}}{n}.
\end{align*}}
By the expansion in equation (\ref{estMLE}) we have that
{\small
\[
\hat{\boldsymbol\theta}_{\alpha_\lambda}
    = \boldsymbol\theta^*_{\alpha_\lambda}
    + \mathbf{H}_{\alpha_\lambda}(\mathbf{y} - b'(\boldsymbol\theta^*_{\alpha_\lambda}))
\]}
asymptotically.  Therefore
{\small
\begin{align*}
AIC_\lambda - \frac{2}{n}\vy^T\vtheta_0 + \frac{2}{n}\vone^Tb(\vtheta_0)
& = L_{KL}(\hat{\vbeta}_\lambda) + \frac{2}{n} (\vy-\vmu)^T (\vtheta_0-\vtheta^*_{\alpha_\lambda})\\
& - \frac{2}{n} ((\vy-\vmu)^T \mH_{\alpha_\lambda}(\vy - \vmu)
- tr\{ (\mX'_{\alpha_\lambda} \mW_{\alpha_\lambda} \mX_{\alpha_\lambda})^{-1}\mX'_{\alpha_\lambda} \mW_0 \mX_{\alpha_\lambda}\}) \\
&
+\frac{2}{n}
(d_{\alpha_\lambda} -
tr\{ (\mX'_{\alpha_\lambda} \mW_{\alpha_\lambda} \mX_{\alpha_\lambda})^{-1}\mX'_{\alpha_\lambda} \mW_0 \mX_{\alpha_\lambda}\} )
+\frac{2}{n} (\vy-\vmu)' (\hat\vtheta_{\alpha_\lambda}-\hat\vtheta_\lambda).
\end{align*}}
Applying Lemmas \ref{thm4Support} and \ref{lemmaEq}, equation (\ref{thm4result}) holds as desired.
\end{proof}

\section{Supplementary Material}
This supplemental section contains the technical details required to show that Theorem 3 can be used to prove the efficiency of $AIC_\lambda$, $GCV_\lambda$, and $AIC_{c_\lambda}$, the regularity conditions required for Theorem 4 to hold, and the mathematical results needed to apply Lemma 2.1 to the simulation examples.

\subsection{Verifying the Conditions of Theorem 3}

The following shows that $AIC_\lambda$, $GCV_\lambda$, and $AIC_{c_\lambda}$ can be written in the form $\tilde{\Gamma}_n(\lambda)$ and that Conditions (C1) and (C2) of Theorem 3 are satisfied.  This implies that the three methods are efficient selectors of the regularization parameter.  \citet{shibata81} and \citet{hurvich89} noted that $AIC$ and $AIC_c$, respectively, can be shown to satisfy these conditions.  We present a detailed argument of these remarks below.

\subsubsection*{$AIC_\lambda$ is Efficient}
Minimizing $AIC_\lambda$ is equivalent to minimizing
{\small
\begin{equation*}
\exp \left( \frac{2d_{\alpha_\lambda}}{n} \right) \hat{\sigma}^2_\lambda.
\end{equation*}}
Using Taylor's expansion we get
{\small
\begin{align*}
\exp \left( \frac{2d_{\alpha_\lambda}}{n} \right) \hat{\sigma}^2_\lambda
& = \sum_{k=0}^\infty \left ( \frac{2d_{\alpha_\lambda}}{n} \right )^k \frac{1}{k!} \\
& = 1 + \frac{2d_{\alpha_\lambda}}{n} + \sum_{k=2}^\infty \left ( \frac{2d_{\alpha_\lambda}}{n} \right )^k \frac{1}{k!},
\end{align*}}
and we see that $AIC_\lambda$ has the same asymptotic properties as
{\small
\begin{equation*}
\tilde{\Gamma}_\lambda = \hat{\sigma}^2_\lambda \left(1+2\frac{d_{\alpha_\lambda}}{n} + \delta_\lambda\right),
\end{equation*}}
where
{\small
\begin{equation*}
\delta_n(\lambda) = \sum_{k=2}^\infty \left ( \frac{2d_{\alpha_\lambda}}{n} \right )^k \frac{1}{k!}.
\end{equation*}}
Therefore, the efficiency of $AIC_\lambda$ can be established by showing that (C1) and (C2) hold.  Consider
{\small
\begin{equation*}
0<\delta_\lambda=\sum_{k=2}^\infty \left ( \frac{2d_{\alpha_\lambda}}{n} \right )^k \frac{1}{k!}
=\exp \left( \frac{2d_{\alpha_\lambda}}{n} \right) - 1 - \frac{2d_{\alpha_\lambda}}{n}.
\end{equation*}}
Therefore, under the assumption that $d_n/n \rightarrow 0$, (C1) is satisfied.  Next consider
{\small
\begin{align*}
0<\frac{\delta_{\lambda}}{\tilde{R}(\hat{\beta}_{\alpha_\lambda})} &= \sum_{k=2}^\infty \left ( \frac{2d_{\alpha_\lambda}}{n} \right )^k \frac{1}{\tilde{R}(\hat{\beta}_{\alpha_\lambda})k!} \\
&\leq \frac{2}{\sigma^2} \sum_{k=2}^\infty \left ( \frac{2d_{\alpha_\lambda}}{n} \right )^{k-1} \frac{1}{k!}
\leq \frac{2}{\sigma^2} \sum_{k=2}^\infty \left ( \frac{2d_n}{n} \right )^{k-1} \frac{1}{(k-1)!} \\
& = \frac{2}{\sigma^2} \sum_{k=1}^\infty \left ( \frac{2d_n}{n} \right )^{k} \frac{1}{k!}
= \frac{2}{\sigma^2}\left( \exp \left( \frac{2d_n}{n} \right) - 1 \right ) \rightarrow 0.
\end{align*}}
Here the inequality on the second line follows from the fact that $R(\hat{\beta}_{\alpha_\lambda})>\sigma^2 d_{\alpha_\lambda}/n$ and the final result follows from the assumption that $d_n/n \rightarrow 0$.  Therefore,
{\small
\begin{equation*}
\sup_{\lambda \in [0,\lambda_{max}]} \frac{|\delta_\lambda|}{L(\hat{\beta}_\lambda)}
= \sup_{\lambda \in [0,\lambda_{max}]} \left | \frac{L(\hat{\beta}_{\alpha_\lambda})}{L(\hat{\beta}_\lambda)}
\frac{\tilde{R}(\hat{\beta}_{\alpha_\lambda})}{L(\hat{\beta}_{\alpha_\lambda})}\frac{\delta_\lambda}{\tilde{R}(\hat{\beta}_{\alpha_\lambda})} \right | \rightarrow_p 0
\end{equation*}}
so (C2) is satisfied.

\subsubsection*{$GCV_\lambda$ is Efficient}
Using Taylor's expansion we get
{\small
\begin{equation*}
 \frac{1}{(1 - d_{\alpha_\lambda}/n)^2} = \sum_{k=1}^\infty k \left(\frac{d_{\alpha_\lambda}}{n}\right)^{k-1}
 = 1 + \frac{2d_{\alpha_\lambda}}{n} + \sum_{k=3}^\infty k \left(\frac{d_{\alpha_\lambda}}{n}\right)^{k-1},
\end{equation*}}
and we see that $GCV_\lambda$ has the same asymptotic properties as
{\small
\begin{equation*}
\tilde{\Gamma}_\lambda = \hat{\sigma}^2_\lambda \left(1+2\frac{d_{\alpha_\lambda}}{n} + \delta_\lambda\right),
\end{equation*}}
where
{\small
\begin{equation*}
\delta_\lambda = \sum_{k=3}^\infty k \left(\frac{d_{\alpha_\lambda}}{n}\right)^{k-1}.
\end{equation*}}
Therefore, the efficiency of $GCV_\lambda$ can be established by showing that (C1) and (C2) hold.  Consider
{\small
\begin{equation*}
0<\delta_\lambda=\sum_{k=3}^\infty k \left(\frac{d_{\alpha_\lambda}}{n}\right)^{k-1}
= \frac{1}{(1 - d_{\alpha_\lambda}/n)^2} - 1 - \frac{2d_{\alpha_\lambda}}{n}.
\end{equation*}}
Therefore, under the assumption that $d_n/n \rightarrow 0$, (C1) is satisfied.  Next consider
{\small
\begin{align*}
0<\frac{\delta_\lambda}{\tilde{R}(\hat{\beta}_{\alpha_\lambda})} &= \sum_{k=3}^\infty k \left(\frac{d_{\alpha_\lambda}}{n}\right)^{k-1} \frac{1}{\tilde{R}(\hat{\beta}_{\alpha_\lambda})} \\
&\leq \frac{1}{\sigma^2}\sum_{k=3}^\infty k \left(\frac{d_{\alpha_\lambda}}{n}\right)^{k-2} \\
& = \frac{1}{\sigma^2} \left ( \sum_{k=3}^\infty (k-1) \left(\frac{d_{\alpha_\lambda}}{n}\right)^{k-2} + \sum_{k=3}^\infty \left(\frac{d_{\alpha_\lambda}}{n}\right)^{k-2} \right )\\
& = \frac{1}{\sigma^2} \left ( \sum_{k=2}^\infty k \left(\frac{d_{\alpha_\lambda}}{n}\right)^{k-1} + \frac{d_{\alpha_\lambda}}{n}\sum_{k=0}^\infty \left(\frac{d_{\alpha_\lambda}}{n}\right)^{k} \right )\\
& =\frac{1}{\sigma^2} \left ( \frac{1}{(1-d_{\alpha_\lambda}/n)^2} -1 + \frac{d_{\alpha_\lambda}/n}{1-d_{\alpha_\lambda}/n}
\right ),
\end{align*}}
which converges to zero uniformly over $\lambda$ under the assumption that $d_n/n \rightarrow 0$.
Here, again, the inequality on the second line follows from the fact that $\tilde{R}(\hat{\beta}_{\alpha_\lambda})>\sigma^2 d_{\alpha_\lambda}/n$.  Therefore,
{\small
\begin{equation*}
\sup_{\lambda \in [0,\lambda_{max}]} \frac{|\delta_\lambda|}{L(\hat{\beta}_\lambda)}
= \sup_{\lambda \in [0,\lambda_{max}]} \left | \frac{L(\hat{\beta}_{\alpha_\lambda})}{L(\hat{\beta}_\lambda)}
\frac{\tilde{R}(\hat{\beta}_{\alpha_\lambda})}{L(\hat{\beta}_{\alpha_\lambda})}\frac{\delta_\lambda}{\tilde{R}(\hat{\beta}_{\alpha_\lambda})} \right | \rightarrow_p 0
\end{equation*}}
so (C2) is satisfied.

\subsubsection*{$AIC_{c_\lambda}$ is Efficient}
We define
{\small
\begin{equation*}
AIC_{c_\lambda} = \log(\hat{\sigma}^2_\lambda) + 2\frac{d_{\alpha_\lambda}+1}{n-d_{\alpha_\lambda}-2}.
\end{equation*}}
This can be equivalently defined as
{\small
\begin{equation*}
AIC_{c_\lambda} = \log(\hat{\sigma}^2_\lambda) + 2\frac{d_{\alpha_\lambda}+1}{n} + 2\frac{(d_{\alpha_\lambda}+1)(d_{\alpha_\lambda}+2)}{n(n-d_{\alpha_\lambda}-2)}.
\end{equation*}}
Based on the second definition of $AIC_{c_\lambda}$ we see that the information criterion has the same asymptotic properties as
{\small
\begin{equation*}
\log(\hat{\sigma}^2_\lambda) + 2\frac{d_{\alpha_\lambda}}{n} + 2\frac{(d_{\alpha_\lambda}+1)(d_{\alpha_\lambda}+2)}{n(n-d_{\alpha_\lambda}-2)},
\end{equation*}}
because they only differ by an additive constant ($2/n$).  Therefore, $AIC_{c_\lambda}$ will have the same asymptotic behavior as
{\small
\begin{equation*}
\exp \left(2\frac{d_{\alpha_\lambda}}{n} + 2\frac{(d_{\alpha_\lambda}+1)(d_{\alpha_\lambda}+2)}{n(n-d_{\alpha_\lambda}-2)} \right) \hat{\sigma}^2_\lambda.
\end{equation*}}
Using Taylor's expansion we get
{\small
\begin{align*}
\exp \left(2\frac{d_{\alpha_\lambda}}{n} + 2\frac{(d_{\alpha_\lambda}+1)(d_{\alpha_\lambda}+2)}{n(n-d_{\alpha_\lambda}-2)} \right)
& = \sum_{k=0}^\infty \left ( 2\frac{d_{\alpha_\lambda}}{n} + 2\frac{(d_{\alpha_\lambda}+1)(d_{\alpha_\lambda}+2)}{n(n-d_{\alpha_\lambda}-2)} \right )^k \frac{1}{k!} \\
& = 1 + \frac{2d_{\alpha_\lambda}}{n} + 2\frac{(d_{\alpha_\lambda}+1)(d_{\alpha_\lambda}+2)}{n(n-d_{\alpha_\lambda}-2)} \\
& + \sum_{k=2}^\infty \left ( 2\frac{d_{\alpha_\lambda}}{n} + 2\frac{(d_{\alpha_\lambda}+1)(d_{\alpha_\lambda}+2)}{n(n-d_{\alpha_\lambda}-2)} \right )^k \frac{1}{k!},
\end{align*}}
and we see that $AIC_{c_\lambda}$ has the same asymptotic properties as
{\small
\begin{equation*}
\tilde{\Gamma}_\lambda = \hat{\sigma}^2_\lambda \left(1+2\frac{d_{\alpha_\lambda}}{n} + \delta_\lambda\right),
\end{equation*}}
where
{\small
\begin{equation*}
\delta_\lambda = 2\frac{(d_{\alpha_\lambda}+1)(d_{\alpha_\lambda}+2)}{n(n-d_{\alpha_\lambda}-2)}+ \sum_{k=2}^\infty \left ( 2\frac{d_{\alpha_\lambda}}{n} + 2\frac{(d_{\alpha_\lambda}+1)(d_{\alpha_\lambda}+2)}{n(n-d_{\alpha_\lambda}-2)} \right )^k \frac{1}{k!}.
\end{equation*}}
Therefore, the efficiency of $AIC_{c_\lambda}$ can be established by showing that (C1) and (C2) hold.  Consider
{\small
\begin{align*}
0<\delta_n(\lambda)
& = 2\frac{(d_{\alpha_\lambda}+1)(d_{\alpha_\lambda}+2)}{n(n-d_{\alpha_\lambda}-2)}+ \sum_{k=2}^\infty \left ( 2\frac{d_{\alpha_\lambda}}{n} + 2\frac{(d_{\alpha_\lambda}+1)(d_{\alpha_\lambda}+2)}{n(n-d_{\alpha_\lambda}-2)} \right )^k \frac{1}{k!} \\
&= 2\frac{(d_{\alpha_\lambda}+1)(d_{\alpha_\lambda}+2)}{n(n-d_{\alpha_\lambda}-2)} +
\exp \left(2\frac{d_{\alpha_\lambda}}{n} + 2\frac{(d_{\alpha_\lambda}+1)(d_{\alpha_\lambda}+2)}{n(n-d_{\alpha_\lambda}-2)} \right) \\
&- 1 - 2\frac{d_{\alpha_\lambda}}{n} - 2\frac{(d_{\alpha_\lambda}+1)(d_{\alpha_\lambda}+2)}{n(n-d_{\alpha_\lambda}-2)},
\end{align*}}
which converges to zero uniformly over $\lambda$ under the assumption that $d_n/n \rightarrow 0$.  Thus, (C1) is satisfied.  Next consider
{\small
{\allowdisplaybreaks
\begin{align*}
0<\frac{\delta_\lambda}{\tilde{R}(\hat{\beta}_{\alpha_\lambda})} &= 2\frac{(d_{\alpha_\lambda}+1)(d_{\alpha_\lambda}+2)}{\tilde{R}(\hat{\beta}_{\alpha_\lambda})n(n-d_{\alpha_\lambda}-2)} \\
&+ \sum_{k=2}^\infty \left ( 2\frac{d_{\alpha_\lambda}}{n} + 2\frac{(d_{\alpha_\lambda}+1)(d_{\alpha_\lambda}+2)}{n(n-d_{\alpha_\lambda}-2)} \right )^k \frac{1}{\tilde{R}(\hat{\beta}^*_{\alpha_\lambda})k!}  \\
&\leq 2\frac{(1+1/d_{\alpha_\lambda})(d_{\alpha_\lambda}+2)}{\sigma^2 (n-d_{\alpha_\lambda}-2)} \\
&+ \frac{n}{\sigma^2d_{\alpha_\lambda}}\sum_{k=2}^\infty \left ( 2\frac{d_{\alpha_\lambda}}{n} + 2\frac{(d_{\alpha_\lambda}+1)(d_{\alpha_\lambda}+2)}{n(n-d_{\alpha_\lambda}-2)} \right )^k \frac{1}{k!}  \\
& \leq 2\frac{(1+1/d_{\alpha_\lambda})(d_{\alpha_\lambda}+2)}{\sigma^2 (n-d_{\alpha_\lambda}-2)} \\
&+ \frac{2}{\sigma^2}\left ( 1 + \frac{(1+1/d_{\alpha_\lambda})(d_{\alpha_\lambda}+2)}{(n-d_{\alpha_\lambda}-2)} \right )\sum_{k=2}^\infty \left ( 2\frac{d_{\alpha_\lambda}}{n} + 2\frac{(d_{\alpha_\lambda}+1)(d_{\alpha_\lambda}+2)}{n(n-d_{\alpha_\lambda}-2)} \right )^{k-1} \frac{1}{k!}\\
& \leq 2\frac{(1+1/d_{\alpha_\lambda})(d_{\alpha_\lambda}+2)}{\sigma^2 (n-d_{\alpha_\lambda}-2)} \\
&+ \frac{2}{\sigma^2}\left ( 1 + \frac{(1+1/d_{\alpha_\lambda})(d_{\alpha_\lambda}+2)}{(n-d_{\alpha_\lambda}-2)} \right )\sum_{k=1}^\infty \left ( 2\frac{d_{\alpha_\lambda}}{n} + 2\frac{(d_{\alpha_\lambda}+1)(d_{\alpha_\lambda}+2)}{n(n-d_{\alpha_\lambda}-2)} \right )^{k} \frac{1}{k!}  \\
& = 2\frac{(1+1/d_{\alpha_\lambda})(d_{\alpha_\lambda}+2)}{\sigma^2 (n-d_{\alpha_\lambda}-2)} \\
& + \frac{2}{\sigma^2}\left ( 1 + \frac{(1+1/d_{\alpha_\lambda})(d_{\alpha_\lambda}+2)}{(n-d_{\alpha_\lambda}-2)} \right )
\left( \exp \left ( 2\frac{d_{\alpha_\lambda}}{n} + 2\frac{(d_{\alpha_\lambda}+1)(d_{\alpha_\lambda}+2)}{n(n-d_{\alpha_\lambda}-2)} \right ) -1 \right),
\end{align*}}}
which converges to zero uniformly over $\lambda$ under the assumption that $d_n/n \rightarrow 0$.
Again, the inequality on the third line follows from the fact that $R(\hat{\beta}^*_n(\alpha_\lambda))>\sigma^2 d_{\alpha_\lambda}/n$.
Therefore,
{\small
\begin{equation*}
\sup_{\lambda \in [0,\lambda_{max}]} \frac{|\delta_\lambda|}{L(\hat{\beta}_\lambda)}
= \sup_{\lambda \in [0,\lambda_{max}]} \left | \frac{L(\hat{\beta}_{\alpha_\lambda})}{L(\hat{\beta}_\lambda)}
\frac{\tilde{R}(\hat{\beta}_{\alpha_\lambda})}{L(\hat{\beta}_{\alpha_\lambda})}\frac{\delta_\lambda}{\tilde{R}(\hat{\beta}_{\alpha_\lambda})} \right | \rightarrow_p 0
\end{equation*}}
so (C2) is satisfied.

\subsection{Regularity Conditions}
Below are the regularity conditions required to derive the properties of the maximum-likelihood estimator for misspecified models.  Refer to \citet{lv2010} for a discussion of these conditions in the context of generalized linear models with no dispersion parameter.
\begin{enumerate}
\item[(R1)] $f_\alpha(y; \boldsymbol\beta)$ is continuous in $\boldsymbol\beta$ for every $\boldsymbol\beta$ in $\Omega$, a compact set of $\mathbb{R}^{d_\alpha}$.
\item[(R2)] (a.) $E_0(\log(g(y)))$ exists and $|\log f_\alpha(y;\boldsymbol\beta)|$ is dominated by an integrable function with respect to g that is independent of $\boldsymbol\beta$. (b.) The KL loss function has a unique minimum at $\boldsymbol\beta^*$, which is an interior point of $\Omega$.
\item[(R3)] (a.) $\partial \log f_\alpha(y;\boldsymbol\beta) / \partial \beta_i$ and $\partial^2 \log f(y;\boldsymbol\beta) / \partial \beta_i \partial \beta_j$ , $i,j = 1, \ldots, d_\alpha$, are measurable functions of $y$ for each $\boldsymbol\beta \in \Omega$ and continuously differentiable functions of $\boldsymbol\beta$ for each $y$.
    (b.) $|\partial \log f_\alpha(y;\boldsymbol\beta) / \partial \beta_i |$, $|\partial \log f(y;\boldsymbol\beta) / \partial \beta_i \partial \beta_j|$, and
    $|(\partial \log f_\alpha(y;\boldsymbol\beta) / \partial \beta_i)(\partial \log f_\alpha(y;\boldsymbol\beta) / \partial \beta_j) |$ are dominated by integrable functions with respect to g, which are independent of $\boldsymbol\beta$.
\item[(R4)]  The matrices
\[
    B(\theta^*) = E_0\left( \frac{\partial \log f_\alpha(y;\boldsymbol\beta)}{\partial \boldsymbol\beta}
                        \frac{\partial \log f_\alpha(y;\boldsymbol\beta)}{\partial \boldsymbol\beta^T}\right)
\]
and
\[
    A(\theta^*) = E_0\left( \frac{\partial^2 \log f_\alpha(y;\boldsymbol\beta)}{\partial \boldsymbol\beta \partial \boldsymbol\beta^T}\right)
\]
are positive definite.
\item[(R5)]
(a.) $\partial^3 \log f_\alpha (y; \boldsymbol\beta)/\partial \beta_i \beta_j \beta_k$ are measurable with respect to $y$ for $i,j,k = 1, \ldots, d_{alpha}$.
(b.) $|\partial \log f_\alpha(y;\boldsymbol\beta) / \partial \beta_i |^2$,
$|\partial^2 \log f_\alpha(y;\boldsymbol\beta) / \partial \beta_i \partial \beta_j |^2$, and
$|\partial^3 \log f_\alpha(y;\boldsymbol\beta) / \partial \beta_i \partial \beta_j \partial \beta_k|^2$ , $i,j,k = 1, \ldots, d_{\alpha}$,
are dominated by integrable functions with respect to g that are independent of $\boldsymbol\beta$.
\item[(R6)]
For some $\delta > 0$, $E ||\mathbf{B}^{-1/2}_n \mathbf{A}_n (\hat{\boldsymbol\beta}_\alpha - \boldsymbol\beta^*_\alpha)||^{3+\delta} = O(1)$, where $\mathbf{A}_n$ is defined as in equation (3) of the manuscript and $\mathbf{B}_n = \mathbf{X}^T_\alpha W_0 \mathbf{X}_\alpha$.
\end{enumerate}

\subsection{Verifying the Conditions of Lemma 2.1}

\subsubsection{Omitted Predictor with Deterministic $\mX$}
We first consider a more general example.  Let the true model be defined as
\[
\vy = \vmu + \boldsymbol\varepsilon,
\]
where $\vy$ is the $n \times 1$ response vector, $\vmu$ is the $n \times 1$ unknown mean vector, and $\boldsymbol \varepsilon$ is a $n \times 1$ noise vector where $\E(\varepsilon_i)=0$ and $\var(\varepsilon_i) = \sigma^2$.  In what follows we assume that
\[
\vmu = \mX \boldsymbol \beta  + \beta_{\text{excl}}\vx_{\text{excl}},
\]
 where $\mathbf{X}$ is a $n \times d_n$ deterministic matrix of predictors, $\boldsymbol \beta$ is a $d_n \times 1$ vector of coefficients, $\vx_{\text{excl}}$ is a $n \times 1$ deterministic vector, and $\beta_{\text{excl}}$ is a constant.  In the following, we take the candidate models to be the least squares regressions based on all $2^{d_n}$ subsets of $\mX$; the predictor $\vx_{\text{excl}}$ is excluded from consideration so that the true model is never included in the set of candidate models.

Assume that the following conditions hold:
\begin{itemize}
\item[(C3)] $\boldsymbol \beta$ contains a fixed number of non-zero entries
\item[(C4)] $\vx_{\text{excl}}$ is orthogonal to the columns of $\mX$
\item[(C5)] $\inf_{n} \frac{\vx_{\text{excl}}^T \vx_{\text{excl}}} {n} > 0$
\end{itemize}

By construction, for any candidate model $\alpha$,
\begin{align*}
nR(\hat\beta_\alpha) & \geq ||\vmu - \mH_{\bar{\alpha}} \vmu||^2 \\
& = ||(\mI - \mH_{\bar{\alpha}})X \vbeta||^2 + \vbeta^T X^T (\mI - \mH_{\bar{\alpha}})  \vx_{\text{excl}} \beta_{\text{excl}}
    + || \vx_{\text{excl}} \beta_{\text{excl}} ||^2 \\
& = ||(\mI - \mH_{\bar{\alpha}})X \vbeta||^2 + || \vx_{\text{excl}} \beta_{\text{excl}} ||^2 \\
& \geq || \vx_{\text{excl}} \beta_{\text{excl}} ||^2 \\
& = n \beta^2_{\text{excl}} \left( \frac{\vx_{\text{excl}}^T \vx_{\text{excl}}} {n} \right) \\
& \geq k_1 n
\end{align*}
for some constant $k_1 > 0$.

For the simulation example in Section 4.1 of the paper, the true vector of coefficients is fixed and trigonometric predictors are used so conditions (C3)-(C5) are satisfied.  Therefore, for that example it follows that $||\vmu - \mH_{\bar{\alpha}} \vmu||^2 \geq k_1 n$ for some constant $k_1 > 0$.

\subsubsection{Exponential Model}
From Fourier analysis (cf. \citet{bloomfield00}), if $n$ is even then
\begin{equation}\label{dft}
\mu_t = e^{4t/n} = A(0)
            + \sum_{0<j<n/2} A(f_j) \cos\left( 2 \pi f_j t \right)
            + \sum_{0<j<n/2} B(f_j) \sin\left( 2 \pi f_j t \right)
            + A(f_{n/2})\cos\left(2 \pi f_{n/2} t\right),
\end{equation}
where $f_j = j/n$,
\[
A(f_j) = \frac{2}{n} \sum_{t=0}^{n-1} \mu_t \cos\left( 2 \pi f_j t \right),
\]
and
\[
B(f_j) = \frac{2}{n} \sum_{t=0}^{n-1} \mu_t \sin\left( 2 \pi f_j t \right).
\]
If $n$ is odd then the rightmost term in (\ref{dft}) is excluded.
To determine $A(f_j)$ and $B(f_j)$ we will use the fact that $d(f_j) = \frac{A(f_j)}{2} - i \frac{B(f_j)}{2}$, where
\[
d(f_j) = \frac{1}{n} \sum_{t=0}^{n-1} \mu_t e^{-2 \pi i f_j t}.
\]
For this example
\[
d(f_j) = \frac{1}{n} \sum_{t=0}^{n-1} e^{4t/n} e^{-2 \pi i f_j t}
     = \frac{ 1 - e^{4-2\pi i }}{n(1 - e^{4/n - 2 \pi i f_j})}
     = \frac{1 - e^{4}}{n}\frac{1}{(1 - e^{4/n}\cos(2\pi f_j))+ i e^{4/n}\sin(2\pi f_j)}.
\]
For any real constants $a$ and $b$,
$
\frac{1}{a + bi} = \frac{a - bi}{a^2 + b^2}.
$
It follows then that
\begin{align*}
d(f_j) & = \frac{1 - e^{4}}{n}
    \frac{1 - e^{4/n}\cos(2\pi f_j)-i e^{4/n}\sin(2\pi f_j)}
    {(1 - e^{4/n}\cos(2\pi f_j))^2 + (e^{4/n}\sin(2\pi f_j))^2} \\
 & = \frac{1 - e^{4}}{n} \left( \frac{1 - e^{4/n}\cos(2\pi f_j)}
    {1 + (e^{4/n})^2 - 2 e^{4/n}\cos(2\pi f_j)}
    - i
     \frac{ e^{4/n}\sin(2\pi f_j)}
    {1 + (e^{4/n})^2 - 2 e^{4/n}\cos(2\pi f_j)} \right).
\end{align*}
Therefore
\[
A(f_j) = 2 \frac{1 - e^{4}}{n} \frac{(e^{-4/n} - \cos(2\pi f_j))}
    {e^{-4/n} + e^{4/n} - 2 \cos(2\pi f_j)}
\]
and
\[
B(f_j) = 2 \frac{1 - e^{4}}{n} \frac{ \sin(2\pi f_j)}
    {e^{-4/n} + e^{4/n} - 2 \cos(2\pi f_j)}.
\]

For a given $d_n$, define the $n \times (n-d_n)$ matrix
$ \mX_{\text{excl}} = ( \vx_{\text{excl}}^1, \vx_{\text{excl}}^2 )$ with components
\begin{equation*}
x^1_{\text{excl}_{tj}} = \sin \left(2\pi t f_j \right)
\end{equation*}
and
\begin{equation*}
x^2_{\text{excl}_{tj}} = \cos \left(2\pi t f_j \right)
\end{equation*}
for $j = d_n/2+1, \ldots , n$.
Based on this notation, the $n \times 1$ mean vector $\vmu$ can be written as
\[
\vmu = \mX \vbeta + \mX_{excl} \vbeta_{excl},
\]
where
\[
\vbeta = [A(0) \; A(f_1) \cdots A(f_{d_n/2}) \; B(f_1) \cdots B(f_{d_n/2})]^T
\]
and
\[
\vbeta_{excl} = [ A(f_{d_n/2 + 1}) \cdots A(f_{n/2}) \; B(f_{d_n/2 + 1}) \cdots B(f_{n/2 - 1})]^T.
\]

For this example, consider
\begin{align*}
nR(\hat\beta_\alpha) & \geq ||\vmu - \mH_{\bar{\alpha}} \vmu||^2 \\
& \geq ||\mX_{excl} \vbeta_{excl}||^2 \\
& \geq \frac{n}{2} \vbeta_{excl}^T \vbeta_{excl} \\
& = \frac{n}{2} \left( \sum_{d_n/2 < j < n/2} A(f_j)^2 + B(f_j)^2 \right) + \frac{n}{2} A(f_{n/2})^2\\
& \geq \frac{n}{2} B(f_{d_n/2 + 1})^2 \\
& = \frac{n}{2} \frac{(2(1-e^4))^2}{n^2} \left( \frac{\sin(2 \pi f_{d_n/2 + 1})}
                                    {e^{-4/n} + e^{4/n} - 2 \cos(2\pi f_{d_n/2 + 1})} \right)^2 \\
& \geq n \frac{c_1}{n^2} \left( \frac{\sin(2 \pi f_{d_n/2 + 1})}
                                    {2(cosh(4/n)-1) + 2(1 -  \cos(2\pi f_{d_n/2 + 1})} \right)^2
\end{align*}
for some positive constant $c_1$.
To simplify notation, define
\[
h_n = \frac{c_1}{n^2} \left( \frac{\sin(2 \pi f_{d_n/2 + 1})}
                                    {2(cosh(4/n)-2) + 2(1 -  \cos(2\pi f_{d_n/2 + 1})} \right)^2.
\]
If $d_n \to \infty$, then $\lim_{n \to \infty} h_n/d_n^2 < \infty$.  It follows that
\[
||\vmu - \mH_{\bar{\alpha}} \vmu||^2 \geq k_1 n d_n^{-2}
\]
for some constant $k_1>0$.

\section*{Acknowledgements}

We would like to thank the Associate Editor and three anonymous referees for comments that helped improve the quality and content of our manuscript greatly.

\bibliography{myreferences}
\bibliographystyle{apalike}

\end{document}